\documentclass{article}
\usepackage{arxiv}

\usepackage{amsthm}

\usepackage{xcolor}  

\usepackage{amssymb}

\usepackage{amsmath}
\usepackage{mathdots}

\usepackage{mathtools}

\usepackage{tensor}

\usepackage{urwchancal}
\DeclareFontFamily{OT1}{pzc}{}
\DeclareFontShape{OT1}{pzc}{m}{it}{<-> s * [1.10] pzcmi7t}{}
\DeclareMathAlphabet{\mathpzc}{OT1}{pzc}{m}{it}

\usepackage{graphicx}
\usepackage{ifpdf}
\ifpdf
\usepackage{epstopdf}
\PrependGraphicsExtensions{.svg}
\fi

\newtheorem{theorem}{Theorem}[section]
\newtheorem{lemma}[theorem]{Lemma}

\newtheorem{proposition}[theorem]{Proposition}

\newtheorem{remark}[theorem]{Remark}

\providecommand{\R}{\mathbb{R}}

\providecommand{\SO}{\mathbf{SO}}

\providecommand{\SE}{\mathbf{SE}}
\providecommand{\SOT}{\mathbf{SOT}}

\providecommand{\grpG}{\mathbf{G}}

\providecommand{\gothg}{\mathfrak{g}}

\providecommand{\gothX}{\mathfrak{X}} 

\providecommand{\so}{\mathfrak{so}}
\providecommand{\se}{\mathfrak{se}}
\providecommand{\sot}{\mathfrak{sot}}

\providecommand{\Sph}{\mathrm{S}}

\providecommand{\calM}{\mathcal{M}}
\providecommand{\calN}{\mathcal{N}}

\providecommand{\calU}{\mathcal{U}}

\providecommand{\vecL}{\mathbb{L}}

\providecommand{\Sym}{\mathbb{S}} 

\providecommand{\tT}{\mathrm{T}} 

\providecommand{\eb}{\mathbf{e}} 

\DeclareMathOperator{\Ad}{Ad}
\DeclareMathOperator{\ad}{ad}

\providecommand{\id}{\mathrm{id}} 

\providecommand{\td}{\mathrm{d}}
\providecommand{\tD}{\mathrm{D}}

\providecommand{\ddt}{\frac{\td}{\td t}}

\providecommand{\mr}[1]{\mathring{#1}} 

\usepackage{accents}
\usepackage{mathtools}
\makeatletter
\providecommand{\scirc}{%
    \hbox{\fontfamily{\rmdefault}\fontsize{0.4\dimexpr(\f@size pt)}{0}\selectfont{\raisebox{-0.52ex}[0ex][-0.52ex]{$\circ$}}}}

\makeatother

\makeatletter
\providecommand{\ucirc}{%
    \hbox{\fontfamily{\rmdefault}\fontsize{0.4\dimexpr(\f@size pt)}{0}\selectfont{\raisebox{0.0ex}[0ex][-0.52ex]{$\circ$}}}}

\makeatother

\mathchardef\mhyphen="2D

\providecommand{\etal}{\textit{et al.~}}

\usepackage[all]{xy}
\usepackage{hyperref}
\usepackage{stfloats}

\providecommand{\inertM}{\calM^\text{\tiny I}}
\providecommand{\inertf}{f^\text{\tiny I}}

\providecommand{\visualM}{\calM^\text{\tiny V}}
\providecommand{\visualN}{\mathcal{N}_n^\text{\tiny  V}(3)}

\providecommand{\vinsG}{\mathbf{SLAM}^\text{\tiny VI}_n(3)}

\providecommand{\vinsT}{\mathcal{T}_n^\text{\tiny VI}(3)}

\providecommand{\vis}{\text{\tiny V}}

\newcommand{\imu}{\text{\tiny B}}

\newcommand{\algorithimRotate}{60}
\newcommand{\hlresult}{\textbf}

\usepackage{wrapfig}
\newcommand{\biographyPhotoWidth}{2.5cm}

\newcommand{\papercitation}{
P. van Goor and R. Mahony, "EqVIO: An Equivariant Filter for Visual-Inertial Odometry," in \textit{IEEE Transactions on Robotics}, vol. 39, no. 5, pp. 3567-3585, Oct. 2023, doi: 10.1109/TRO.2023.3289587.
}

\newcommand{\publicationdetails}
{
\papercitation \\
\copyrightNoticeIEEE{2023} 
}
\newcommand{\publicationversion}
{Author accepted version}

\begin{document}


\title{EqVIO: An Equivariant Filter for \\ Visual Inertial Odometry}
\headertitle{EqVIO: An Equivariant Filter for Visual Inertial Odometry}

\thanks{P. van Goor and R. Mahony were with the Systems Theory and Robotics group at the School of Engineering, Australian National University.
e-mail: \texttt{first\_name.last\_name@anu.edu.au}.}

\author{
\href{https://orcid.org/0000-0003-4391-7014}{\includegraphics[scale=0.06]{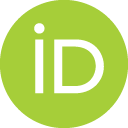}\hspace{1mm}
Pieter van Goor}
\\
    Systems Theory and Robotics Group \\
    Australian Centre for Robotic Vision \\
	Australian National University \\
    ACT, 2601, Australia \\
    \texttt{Pieter.vanGoor@anu.edu.au} \\
	\And	\href{https://orcid.org/0000-0002-7803-2868}{\includegraphics[scale=0.06]{orcid.png}\hspace{1mm}
    Robert Mahony}
\\
    Systems Theory and Robotics Group \\
    Australian Centre for Robotic Vision \\
	Australian National University \\
    ACT, 2601, Australia \\
	\texttt{Robert.Mahony@anu.edu.au} \\
}


\maketitle


\begin{abstract}
Visual-Inertial Odometry (VIO) is the problem of estimating a robot's trajectory by combining information from an inertial measurement unit (IMU) and a camera, and is of great interest to the robotics community.
This paper develops a novel Lie group symmetry for the VIO problem and applies the recently proposed equivariant filter.
The proposed symmetry is compatible with the invariance of the VIO reference frame, leading to improved filter consistency.
The bias-free IMU dynamics are group-affine, ensuring that filter linearisation errors depend only on the bias estimation error and measurement noise.
Furthermore, visual measurements are equivariant with respect to the symmetry, enabling the application of the higher-order equivariant output approximation to reduce approximation error in the filter update equation.
As a result, the equivariant filter (EqF) based on this Lie group is a consistent estimator for VIO with lower linearisation error in the propagation of state dynamics and a higher order equivariant output approximation than standard formulations.
Experimental results on the popular \textit{EuRoC} and \textit{UZH FPV} datasets demonstrate that the proposed system outperforms other state-of-the-art VIO algorithms in terms of both speed and accuracy.
\end{abstract}


\section{Introduction}


Visual Inertial Odometry (VIO) is the problem of determining the trajectory of a robot from a combination of a camera and an inertial measurement unit (IMU).
This problem is of enduring interest to the robotics community due to the ubiquity of systems where such sensors are available, including smartphones, VR/AR headsets, racing drones, and more.
Solutions to the `standard' variant of VIO, where only a single monocular camera is available, are of particular importance due their wide range of applications.
Additionally, the IMU complements the visual data from a monocular camera by providing high-rate motion predictions and making the scale of the system observable, thereby overcoming a key weakness of the camera-only Visual Odometry (VO) problem.

State-of-the-art solutions for VIO are based on either the extended Kalman filter (EKF) or sliding-window optimisation.
EKF-based solutions, such as ROVIO \cite{2017_bloesch_IteratedExtendedKalman}, OpenVINS \cite{2020_geneva_OpenVINSResearchPlatform}, and MSCKF \cite{2007_mourikis_MultistateConstraintKalman}, are generally less accurate than optimisation-based methods but have lower compute and memory requirements and tend to be used in highly dynamic embedded systems applications such as VR headsets, smartphone applications, aerial vehicles, etc.
On the other hand, optimisation-based methods, such as VINS-mono \cite{2018_qin_VinsmonoRobustVersatile} and OKVIS \cite{2015_leutenegger_KeyframebasedVisualInertial}, tend to be more accurate than EKF-based methods but require significant compute and memory resources making them appropriate in applications such as automotive, larger robotic systems, etc.
The main cause of loss of accuracy for EKF methods relative to optimisation-based methods is associated with the accumulation of linearisation errors.
Recent advances in the theory of equivariant systems \cite{2018_barrau_InvariantKalmanFiltering,2020_mahony_EquivariantSystemsTheory} have shown that exploiting the Lie group symmetries of a system can lead to improved filter designs such as the invariant EKF (IEKF) and the Equivariant Filter (EqF) \cite{2023_vangoor_EquivariantFilterEqF} that minimize linearisation error.

In this paper, we develop a novel Lie group for the VIO problem, and exploit this symmetry in the implementation of an \emph{equivariance-based} VIO algorithm we term \emph{EqVIO}.
Unlike EKF designs, the EqF back-end of our proposed system has out-of-the-box consistency properties, exact linearisation of the bias-free IMU error dynamics, and a better (higher-order) linearisation of the visual measurement function.
The advantages of these properties are made clear in the experimental results, where EqVIO outperforms state-of-the-art EKF- and optimisation-based algorithms in terms of both the accuracy of the estimated trajectory and the speed of processing each frame.

The key contributions of this paper are as follows.
\begin{itemize}
\item A novel Lie group, the \emph{VI-SLAM group}, is developed for the VIO problem.
This Lie group symmetry is compatible with the reference frame invariance of VIO.
Additionally, in contrast to the symmetries explored in prior literature \cite{2018_brossard_UnscentedKalmanFilter}, the visual measurement function of VIO is equivariant with respect to the VI-SLAM group.

\item The advantages of the VI-SLAM group are clearly demonstrated.
Its $\SE_2(3)$ component is shown to eliminate linearisation error in the bias-free IMU error dynamics; the only error in the propagation of the full IMU error dynamics is due to the measurement noise and bias estimation error.
The novel landmark symmetry, based on $\SOT(3)$ components, eliminates second-order approximation error of the visual measurement function by exploiting equivariance.
It also improves on and explains the well-known advantages of the inverse-depth parametrisation of landmarks that is core to modern filter performance in VIO algorithms.

\item A novel VIO algorithm, \emph{EqVIO}, is proposed that combines a simple feature-tracking front-end and basic outlier rejection with an equivariant filter implementation.
EqVIO is shown to outperform state-of-the-art VIO algorithms in both speed and accuracy on both the popular EuRoC \cite{2016_burri_EuRoCMicroAerial} and the challenging UZH FPV \cite{2019_delmerico_AreWeReady} datasets.
Our implementation of EqVIO is open source and publicly available under a GNU GPLv3 licence%
\footnote{\url{https://github.com/pvangoor/eqvio}}.
\end{itemize}

The proposed filter is in the class of invariant and equivariant observer designs.
Analogous to previous IEKF solutions, it exploits the same $\SE_2(3)$ symmetry for estimation of the IMU position, attitude and velocity, and the provided estimate is statistically consistent.
This is unsurprising since the EqF specialises to the IEFK when the system state can be identified with the symmetry Lie group and specific local coordinates are chosen \cite[Appendix B]{2023_vangoor_EquivariantFilterEqF}.
However, the EqF framework used in this paper enables the application of $\SOT(3)$ as a symmetry to landmark estimation, which provides a powerful third-order approximation of the visual measurements, and cannot be applied in an IEKF setting.

This work is an extension of \cite{2021_vangoor_EquivariantFilterVisual}, and improves over the previous version by; providing online calibration of IMU-camera extrinsics; including the estimation of robot pose directly in the EqF rather than as a separate bundle adjustment step; detailing the effect of symmetry on the linearisation of IMU error dynamics and visual measurements; and greatly expanding the experimental results to include more thorough comparisons with other state-of-the-art algorithms and demonstrate filter consistency.

\section{Related Work}

\subsection{Visual-Inertial Odometry}

Although most VIO solutions rely on constructing a map of the robot's local environment, the accuracy of this map is not considered important in evaluating system performance, in contrast to traditional Simultaneous Localisation and Mapping (SLAM).
Some of the first systems to focus on the problem of trajectory estimation from stereo or monocular vision data, as distinct from general SLAM, were proposed in \cite{2004_nister_VisualOdometry,2006_nister_VisualOdometryGround,2007_klein_ParallelTrackingMapping,2007_davison_MonoSLAMRealtimeSingle}.
An important milestone in the development of VIO systems is the Multi-State Constrained Kalman Filter (MSCKF) \cite{2007_mourikis_MultistateConstraintKalman}, which approached the problem by applying a fixed-lag EKF, and, notably, eliminating the estimation of landmark positions from the filter process.
This modification resulted in an efficient algorithm for VIO with only linear complexity in the number of landmarks considered.
In \cite{2010_konolige_LargescaleVisualOdometry}, Konolige \etal considered a system equipped with a stereo camera and an IMU, and employed bundle adjustment to solve the VIO problem.
They improved their results by using specialised image features, and discussed the challenge of using traditional image features in self-similar outdoor environments.

Since 2015, the monocular VIO problem specifically has seen substantial interest in the robotics community.
Bloesch \etal \cite{2017_bloesch_IteratedExtendedKalman} developed \emph{ROVIO}: a VIO algorithm that mixes an arbitrary number of cameras with IMU measurements in an iterated EKF framework.
In contrast to the majority of EKF-based VIO systems, ROVIO used a `direct' error formulation; that is, rather than obtaining feature coordinates from an image, the image pixel values were considered directly in the system model of the EKF.
In parallel, Leutenegger \etal \cite{2015_leutenegger_KeyframebasedVisualInertial} developed \emph{OKVIS}, which solves monocular and stereo VIO by using non-linear optimisation on a sliding window of `keyframes'.
Semi-direct Visual Odometry (SVO) \cite{2017_forster_SVOSemidirectVisual} used a sparse set of image patches in frame-to-frame optimisation to greatly reduce the presence of outliers and to solve VO at very high speeds.
While SVO is not strictly a VIO solution, as it does not use an IMU, Delmerico and Scaramuzza \cite{2018_delmerico_BenchmarkComparisonMonocular} propose two methods to combine the output from SVO with an IMU in a filtering or smoothing framework.
Recently, Qin \etal \cite{2018_qin_VinsmonoRobustVersatile} combined a range of modern SLAM techniques to develop VINS-MONO, which performs tightly coupled keyframe optimisation-based VIO with efficient loop closure, and achieves competitive accuracy on popular datasets.
Delmerico and Scaramuzza \cite{2018_delmerico_BenchmarkComparisonMonocular} benchmarked a range of state-of-the-art VIO systems.
They showed that, generally, the VIO systems optimised for speed and CPU usage suffered from relatively low accuracy, and that many of the popular systems tend to fail on challenging datasets or on limited hardware.
OpenVINS \cite{2020_geneva_OpenVINSResearchPlatform} is another recent VIO system, which mixes the MSCKF \cite{2007_mourikis_MultistateConstraintKalman} with a traditional EKF and achieves state-of-the-art performance.
The primary contribution of OpenVINS was to provide a well-documented open platform for EKF-based VIO research.

In summary, the literature on VIO is split between EKF-based and optimisation-based algorithms.
EKF-based algorithms are preferred for their efficiency when computational resources are constrained, but suffer from linearisation error that accumulates and degrades performance over time.

\subsection{Equivariant Observers for VIO}

Equivariant observers are state estimators that exploit available Lie group symmetries of a given problem.
Two key examples include the invariant EKF (IEKF) \cite{2018_barrau_InvariantKalmanFiltering} and the Equivariant Filter (EqF) \cite{2023_vangoor_EquivariantFilterEqF}.
The success of equivariant observers in other robotics problems has led several authors to investigate their application to inertial navigation, SLAM, and VIO.
In \cite{2014_barrau_InvariantParticleFiltering}, Barrau and Bonnabel proposed the extended Special Euclidean group $\SE_2(3)$, and show that it can be used to obtain an exact linearisation of IMU error dynamics when the biases are known.
This represents a clear improvement over the common representation of IMU states in the SLAM and VIO literature, which uses an on-manifold EKF \cite{2013_hertzberg_IntegratingGenericSensor} to obtain a minimal representation of rotation error between quaternions that is analogous to the well-known multiplicative EKF (MEKF) \cite{2003_markley_AttitudeErrorRepresentations}.

In the earliest work examining symmetry properties of the SLAM problem, Barrau and Bonnabel \cite{2016_barrau_EKFSLAMAlgorithmConsistency} proposed a novel class of Lie groups, $\SE_n(m)$, and showed that this is a symmetry suitable for the classical SLAM problem.
They further showed that this symmetry is compatible with the reference frame invariance of SLAM, and that the resulting IEKF consequently overcomes the well-known consistency issues of EKF-based SLAM \cite{2006_bailey_ConsistencyEKFSLAMAlgorithm}.
In \cite{2017_zhang_ConvergenceConsistencyAnalysis}, Zhang \etal performed an observability analysis of the IEKF for SLAM, and compared its performance to a range of other EKFs for SLAM in simulation.
Wu \etal \cite{2017_wu_InvariantEKFVINSAlgorithm} then combined $\SE_n(m)$ with the MSCKF concept of \cite{2007_mourikis_MultistateConstraintKalman} to propose an invariant MSCKF for VIO, which they showed to be a consistent filter.
They contrasted this to the original MSCKF, which exhibited growing inconsistency in a series of Monte Carlo simulation trials.
Brossard \etal \cite{2018_brossard_InvariantKalmanFiltering} derived an invariant unscented Kalman filter (IUKF) for monocular SLAM using the Lie group proposed in \cite{2016_barrau_EKFSLAMAlgorithmConsistency}, and outperformed other invariant filters for VIO.
A number of other works \cite{2018_heo_ConsistentEKFBasedVisualInertial,
2018_heo_ConsistentEKFBasedVisualInertiala,
2018_brossard_UnscentedKalmanFilter}
have also explored applying variants of the IEKF to VIO in an MSCKF framework.
Recently, Yang \etal \cite{2022_yang_DecoupledRightInvariant} coupled the $\SE_2(3)$ symmetry with the `first-estimates Jacobian' technique of \cite{2008_huang_AnalysisImprovementConsistency}.
They demonstrated improved accuracy and consistency over other filter-based algorithms implemented on the OpenVINS platform \cite{2020_geneva_OpenVINSResearchPlatform}, evaluated in simulation and on the TUM-VI dataset \cite{2018_schubert_TUMVIBenchmark}.

Recently, the authors \cite{2021_vangoor_ConstructiveObserverDesign} developed a novel Lie group for visual SLAM under which the visual measurements of landmarks are equivariant, unlike in the previously explored $\SE_m(n)$ symmetry.
The IEKF cannot be directly applied using this symmetry as the Lie group is of a higher dimension than the underlying state space.
This issue is overcome by the recent EqF \cite{2023_vangoor_EquivariantFilterEqF}, which additionally provides a framework for exploiting the equivariance of a system output function to reduce linearisation error.
To the best of our knowledge, there has been no equivariant observer applied to VIO with a symmetry that is compatible with visual measurements prior to this paper and its previous version \cite{2021_vangoor_EquivariantFilterVisual}.

\subsection{Parametrisations of VIO Landmarks}

The representation of the robot pose and environment map are known to have a significant impact on the accuracy of EKF-based SLAM and VIO approaches.
In \cite{2004_castellanos_LimitsConsistencyEKFbased}, Castellanos \etal identified the inconsistency of the EKF for SLAM when using a straightforward inertial-frame representation of landmarks.
They proposed to use a Euclidean body-fixed representation of landmarks instead, and showed that this improved the consistency of the EKF.
Another key work in understanding the impact of landmark representations in EKF SLAM is by Civera \etal \cite{2008_civera_InverseDepthParametrization}, who proposed the earliest version of the \emph{inverse-depth} parametrisation of landmarks.
The key advantage in this representation is that it is able to represent a large uncertainty in the distance of a landmark from the robot's initial position using a Gaussian distribution.
Sol\`{a} \cite{2010_sola_ConsistencyMonocularEKFSLAM} investigated a variety of landmark parametrisations and proposed an anchored homogeneous point representation similar to the inverse-depth parametrisation.
He showed that this representation provided more consistent estimation in a SLAM system compared to inverse-depth in a series of simulation experiments.
However, further comparisons by Sol\`{a} \etal \cite{2011_sola_ImpactLandmarkParametrization} showed similar performance between the inverse-depth and anchored homogeneous point representations, and concluded that the inverse-depth parametrisation is preferred for its lower computational cost.
Li \etal \cite{2013_li_HighprecisionConsistentEKFbased} also investigated the anchored homogeneous point and inverse-depth representations, showing that the first yielded better filter consistency while the second resulted in better absolute filter performance.
A more recent version of the inverse-depth parametrisation is presented by Bloesch \etal \cite{2017_bloesch_IteratedExtendedKalman}, who adapted the on-manifold EKF approach developed in \cite{2013_hertzberg_IntegratingGenericSensor} to obtain a minimal representation of unit vectors on the sphere.
While the existing inverse-depth parametrisations have been empirically shown to improve performance over the Euclidean parametrisation, the cause of this improvement is not well characterised.
The discovery of symmetries for visual measurements \cite{2021_vangoor_ConstructiveObserverDesign} motivates the development of a new compatible parametrisation.

\subsection{Consistency of Filter-Based VIO}

Stochastic filters for SLAM and VIO, such as the EKF, are said to be consistent if the probability distribution they report matches the true system statistics.
The straightforward EKF solution for SLAM was shown to be inconsistent in experiments carried out in \cite{2004_castellanos_LimitsConsistencyEKFbased}.
A number of authors have developed advanced modifications to the standard EKF to reduce inconsistency, and these techniques continue to be used in state-of-the-art VIO systems such as OpenVINS \cite{2020_geneva_OpenVINSResearchPlatform}.
However, significant progress has recently been made by new VIO solutions exploiting Lie group symmetries in their designs to circumvent the consistency problem entirely.
Huang \etal \cite{2008_huang_AnalysisImprovementConsistency} demonstrated that the inconsistency of standard EKF SLAM is associated with a mismatch between the observability of the linearised error-state system and the observability of the true system; the true system has a six-dimensional unobservable subspace corresponding to transformations of the reference frame, while the linearised system does not.
They proposed a first estimates Jacobian (FEJ) which overcomes this observability issue with only a minor loss of accuracy in the EKF.
Li and Mourikis \cite{2013_li_HighprecisionConsistentEKFbased} applied the FEJ concept to an MSCKF design and showed that this improved consistency over a standard MSCKF.
In \cite{2014_hesch_ConsistencyAnalysisImprovement}, Hesch \etal studied the observability of VIO specifically, and identified the same consistency issues in standard EKF solutions.
They developed an observability constrainted (OC) EKF for VIO that directly enforces the unobservable directions of the system in the update step of the EKF, and showed that this significantly reduces filter inconsistency.
Barrau and Bonnabel \cite{2016_barrau_EKFSLAMAlgorithmConsistency} proposed a novel Lie group $\SE_{n+1}(3)$, and show that an IEKF design for SLAM with this symmetry provides a consistent estimator.
They showed that the linearised system admits the same unobservable directions as the true system, due to the compatibility of the proposed Lie group and the reference frame invariance of SLAM.
This same symmetry was used by Wu \etal \cite{2017_wu_InvariantEKFVINSAlgorithm} to develop an IEKF for VIO.
This was also shown to provide a consistent filter for VIO, as the symmetry respects the invariance of the VIO problem to changes in the reference frame yaw and position.
Recently, Huai and Huang \cite{2022_huai_RobocentricVisualInertial} formulated the VIO problem with respect to a non-global moving frame, and showed that an algorithm based on this approach does not suffer from the unobservability mismatch between the true and linearised systems.
Finally, Yang \etal \cite{2022_yang_DecoupledRightInvariant} applied an IEKF to VIO using the FEJ technique and showed that it outperforms a standard FEJ-EKF design.
Recent developments have studied the VIO problem from the perspective of Lie group symmetries; they show that designing a filter that exploits these symmetries can overcome the observability mismatch of standard EKF designs and yield a consistent solution for VIO.
It is clear that a key advantage of algorithms based on invariant and equivariant principles is strong consistency properties.

\section{Mathematical Preliminaries}

For a comprehensive introduction to smooth manifolds and Lie groups, the authors recommend \cite{2013_lee_IntroductionSmoothManifolds}.

\subsection{Smooth Manifolds}

Given a smooth manifold $\calM$, denote the tangent space at $\xi \in \calM$ by $\tT_\xi \calM$.
The tangent bundle of $\calM$ is written $\tT \calM$.
If $f : \calM \to \calN$ is a differentiable function between smooth manifolds, the differential of $f$ with respect to $\zeta$ at a point $\xi \in \calM$ is
\begin{align*}
    \tD_\zeta |_\xi f(\zeta) : \tT_\xi \calM &\to \tT_{f(\xi)} \calN, \\
     u &\mapsto \tD_\zeta |_\xi f(\zeta)[u].
\end{align*}
When the base point is left unspecified, the differential of $f$ is a map between tangent bundles $\tD f : \tT \calM \to \tT \calN$.

Given $f : \calM \to \calN$ and $g : \calN \to \calN'$, write the composition $f \circ g : \calM \to \calN'$.
When $f$ and $g$ are linear maps we may also write $f \cdot g$.

Denote the $n$-sphere as
\begin{align*}
    \Sph^n := \{ x \in \R^n \mid \vert x \vert = 1 \}.
\end{align*}
The 1-sphere $\Sph^1$ is simply the circle, and forms a Lie group under addition of angles.
For any $n$, the projection onto the sphere $\pi_{\Sph^n} : \R^n \setminus \{0\} \to \Sph^n$ is defined to be
\begin{align}
    \pi_{\Sph^n}(x) := \frac{x}{\vert x \vert}.
    \label{eq:sphere-projector}
\end{align}

\subsection{Lie Group Theory}

For a Lie group $\grpG$, we write the Lie algebra as $\gothg$.
The identity is denoted $\id \in \grpG$, and Left and right translation are written
\begin{align*}
    L_X(Y) &:= X Y, &
    R_X(Y) &:= Y X,
\end{align*}
respectively.
The exponential map is written $\exp$, and its inverse (when defined) is the logarithmic map $\log$.
The Adjoint maps $\Ad : \grpG \times \gothg \to \gothg$ and $\ad : \gothg \times \gothg \to \gothg$ are defined by
\begin{align*}
    \Ad_X U &= \tD L_X \tD R_{X^{-1}} U, &
    \ad_U V &= [U, V].
\end{align*}
where $[\cdot,\cdot]$ is the Lie-bracket on $\gothg$.
The \emph{wedge} and \emph{vee} operators are linear isomorphisms
\begin{align*}
    \cdot^\wedge &: \R^{\dim \gothg} \to \gothg, &
    \cdot^\vee &: \gothg \to \R^{\dim \gothg},
\end{align*}
satisfying $(u^\vee)^\wedge = u$ for all $u \in \gothg$.
When it is not clear from context, we will use a subscript to indicate which Lie group or algebra a particular operation is associated with.

A right group action of a Lie group $\grpG$ on smooth manifold $\calM$ is a smooth map $\phi : \grpG \times \calM \to \calM$ satisfying
\begin{align}
    \label{eq:group_action_compatible}
    \phi(XY, \xi) &= \phi(Y, \phi(X, \xi)), \\
    \label{eq:group_action_identity}
    \phi(\id, \xi) &= \xi,
\end{align}
for all $X, Y \in \grpG$ and $\xi \in \calM$.
A left group action is defined similarly, except that the compatibility condition \eqref{eq:group_action_compatible} is reversed.

A product Lie group is formed from the combination of multiple existing Lie groups.
If $\grpG_1, ..., \grpG_n$ are Lie groups, then the product Lie group is
\begin{align} \label{eq:product_lie_group}
    \grpG_1 \times \cdots \times \grpG_n := \{ (X_1,...,X_n) \mid X_i \in \grpG_i \},
\end{align}
with multiplication, identity, and inverse given by
\begin{align*}
    (X_1,...,X_n)(Y_1,...,Y_n) &:= (X_1 Y_1,...,X_n Y_n), \\
    \id_{\grpG_1 \times \cdots \times \grpG_n} &:= (\id_{\grpG_1}, ..., \id_{\grpG_n}), \\
    (X_1,...,X_n)^{-1} &:= (X_1^{-1},...,X_n^{-1}).
\end{align*}

\subsection{Important Lie Groups}

The \emph{special orthogonal group} is the set of 3D rotations,
\begin{align*}
    \SO(3) &:= \{ R \in \R^{3 \times 3} \mid R^\top R = I_3, \; \det(R) = 1 \}, \\
    \so(3) &:= \{ \omega^\times \mid \omega \in \R^3 \}, \\
    \omega^\times &:= \begin{pmatrix}
        0 & -\omega_3 & \omega_2 \\
        \omega_3 & 0 & -\omega_1 \\
        -\omega_2 & \omega_1 & 0
    \end{pmatrix}.
\end{align*}
Note that $\omega^\times$ is the unique $3\times 3$ matrix satisfying $\omega^\times v = \omega \times v$ for all $v \in \R^3$.
For any $\omega \in \R^3$, define $(\omega)_{\so(3)}^\wedge := \omega^\times$.

The \emph{scaled orthogonal transforms} are
\begin{align*}
    \SOT(3) &:= \left\{ \begin{pmatrix}
        R & 0 \\ 0 & c
    \end{pmatrix} \; \middle\vert \; R \in \SO(3), \; c > 0 \right\}, \\
    \SOT(3) &:= \left\{ (\Omega, s)^\wedge_{\sot(3)} \; \middle\vert \; \Omega \in \R^3, \; s \in \R \right\}, \\
    (\Omega, s)^\wedge_{\sot(3)} &:= \begin{pmatrix}
        \Omega^\times & 0 \\ 0 & s
    \end{pmatrix}.
\end{align*}
Elements of $\SOT(3)$ may be written as $Q = (R_Q, c_Q)$, where $R_Q \in \SO(3)$ and $c_Q > 0$.
Given $Q \in \SOT(3)$ and $p \in \R^3$, we use the shorthand\footnote{This shorthand notation corresponds to the normal homogeneous coordinates notation.} $Q p := c_Q R_Q p$.

The \emph{special Euclidean group} is the set of 3D rigid body poses,
\begin{align*}
    \SE(3) &:= \left\{ \begin{pmatrix}
        R & x \\ 0 & 1
    \end{pmatrix} \; \middle\vert \; R \in \SO(3), \; x \in \R^3 \right\}, \\
    \se(3) &:= \left\{ (\Omega, v)^\wedge_{\se(3)} \; \middle\vert \; \Omega, v \in \R^3 \right\}, \\
    (\Omega, v)^\wedge_{\se(3)} &:= \begin{pmatrix}
        \Omega^\times & v \\ 0 & 0    \end{pmatrix}.
\end{align*}
Elements of $\SE(3)$ may be written as $P = (R_P, x_P)$, where $R_P \in \SO(3)$ and $x_P \in \R^3$.
Given $P \in \SE(3)$ and $p \in \R^3$, we frequently use the shorthand $P p := R_P p + x_P$ to denote the standard left group action of $\SE(3)$ on $\R^3$.

The \emph{extended special Euclidean group} \cite{2014_barrau_InvariantParticleFiltering} is
\begin{align*}
    \SE_2(3) &:= \left\{ \begin{pmatrix}
        R & x & v \\ 0 & 1 & 0 \\ 0 & 0 & 1
    \end{pmatrix} \; \middle\vert \; R \in \SO(3), \; x, v \in \R^3 \right\}, \\
    \se_2(3) &:= \left\{ (\Omega, v, a)^\wedge_{\se_2(3)} \; \middle\vert \; \Omega, v, a \in \R^3 \right\}, \\
    (\Omega, v, a)^\wedge_{\se_2(3)} &:= \begin{pmatrix}
        \Omega^\times & v & a \\ 0 & 0 & 0 \\ 0 & 0 & 0
    \end{pmatrix}
\end{align*}
Elements of $\SE_2(3)$ may be written as $A = (R_A, x_A, v_A)$, or as $A = (P_A, v_A)$, where $R_A \in \SO(3)$, $x_A \in \R^3$, $v_A \in \R^3$, and $P_A = (R_A, x_A) \in \SE(3)$.

\section{Problem Description}

Choose an arbitrary inertial reference frame $\{0\}$, and consider a robot equipped with an IMU and a camera, both of which are rigidly attached to the robot's body.
Label the camera frame $\{C\}$, and identify the IMU frame $\{I\}$ with the body-fixed frame of the robot $\{B\}$.
The states of the Visual-Inertial SLAM (VI-SLAM) problem are modelled as follows.
\begin{itemize}
    \item $P_\imu := (R_\imu, x_\imu) \in \SE(3)$ is the pose of the IMU $\{B\}$ with respect to the inertial frame $\{0\}$,
    \item $v_\imu \in \R^3$ is the linear velocity of the IMU expressed in the inertial frame $\{0\}$,
    \item $b_\imu = (b_\imu^\Omega, b_\imu^a) \in \R^6$ is the IMU bias (where $b_\imu^\Omega$ and $b_\imu^a$ are the gyroscope and accelerometer biases, respectively),
    \item $T = (R_T, x_T) \in \SE(3)$ is the pose of the camera $\{C\}$ with respect to the body-fixed frame $\{B\}$
    \item $p_i \in \R^3$ are the coordinates of landmark $i$ in the inertial frame $\{0\}$.
\end{itemize}
Figure \ref{fig:vislam_diagram} shows a diagram of the full VI-SLAM configuration.

\begin{figure}[!htb]
    \centering
    \includegraphics[width=0.6\linewidth]{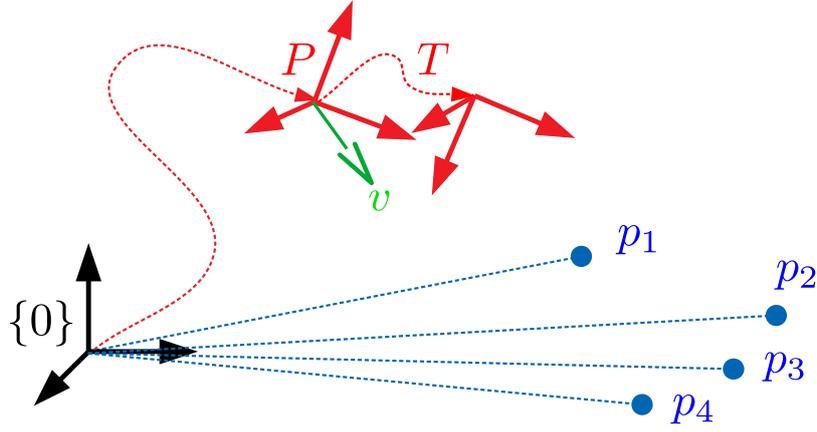}
    \caption{A diagram showing the states of visual-inertial SLAM. Note that the IMU biases are excluded here.}
    \label{fig:vislam_diagram}
\end{figure}

Let $\inertM := \SE(3) \times \R^3$ denote the \emph{navigation state space} and define the \emph{navigation state} $\xi_\imu := (P_\imu, v_\imu) \in \inertM$.
We frequently use the notation $\xi := (\xi_\imu, b_\imu, T, p_i)$ as shorthand for the full VI-SLAM state.
To ensure that the visual measurements are always well defined we assume that the system trajectory considered never passes through an exception set,
\begin{multline*}
    \mathcal{E} := \{ (\xi_\imu, b_\imu, T, p_i) \in \inertM \times \SE(3) \times (\R^3)^n \mid \\
    (P_\imu T)^{-1} p_i = 0 \text{ for any } i \},
\end{multline*}
corresponding to all situations where the camera centre coincides with a landmark point.
To formalise this, we define the visual-inertial SLAM (VI-SLAM) total space
\[
\vinsT := \inertM \times \SE(3) \times (\R^3)^n - \mathcal{E}
\]
and consider the visual-inertial SLAM problem on $\vinsT$.
Note that $\vinsT$ is an open subset of a smooth manifold and as such is itself a smooth manifold.

\subsection{VI-SLAM Dynamics}

Let the acceleration due to gravity in the inertial frame $\{ 0 \}$ be $g \eb_3$, where $g \approx 9.81$ m/s$^2$ and $\eb_3 \in \Sph^2$ is the standard direction of gravity in the inertial frame.
Let the measured angular velocity and linear acceleration obtained from the IMU be $(\Omega, a) \in \vecL := \R^3 \times \R^3$, where $\vecL$ is the \emph{input space}.
Then the VI-SLAM dynamics $f : \vecL \to \gothX(\vinsT)$ are
\begin{align}
    \dot{\xi} &= f_{(\Omega, a)}(\xi);    \label{eq:vins_dynamics}
    &\dot{R}_\imu &= R_\imu (\Omega - b^\Omega_\imu)^\times, \\
    &&\dot{x}_\imu &= v_\imu, \notag \\
    &&\dot{v}_\imu &= R_\imu (a - b^a_\imu) + g \eb_3, \notag \\
    &&\dot{b}^\Omega_\imu &= 0, \notag \\
    &&\dot{b}^a_\imu &= 0, \notag \\
    &&\dot{T} &= 0, \notag \\
    &&\dot{p}_i &= 0. \notag
\end{align}

\subsection{VI-SLAM Measurements}

The camera measurements are modelled as $n$ bearing measurements of the landmarks $p_i$ in the camera frame $\{C\}$ on the manifold $\visualN := (\Sph^2)^n$ where the superscript ``V'' stands for visual measurements.
The measurement function $h : \vinsT \to \visualN$ is given by
\begin{align}
    h(\xi) &:= \left( h^1(\xi), ... h^n(\xi) \right), \label{eq:measurement_function} \\
    h^k(\xi_\imu, b_\imu, T, p_i) &:= \pi_{\Sph^2} \left( (P_\imu T)^{-1} (p_k) \right), \notag
\end{align}
where $\pi_{\Sph^2}$ is defined as in \eqref{eq:sphere-projector}.
Modelling the bearing measurements directly on the sphere rather than the image plane enables the proposed system to model a wide variety of monocular cameras.
Consideration of different camera models is omitted from this paper but is included in the EqVIO software package.

\subsection{Invariance of Visual-Inertial SLAM}
\label{sec:invariance}

Let $\eb_3$ be the standard gravity direction and define the semi-direct product group
\begin{align*}
    \Sph^1 \ltimes_{\eb_3} \R^3 := \{
        (\theta, x) \; \vline \; \theta \in \Sph^1, x \in \R^3
    \},
\end{align*}
with group product, identity and inverse
\begin{align*}
    (\theta^1, x^1) \cdot (\theta^2, x^2) &= (\theta^1 + \theta^2, x^1 + R_{\eb_3}(\theta^1) x^2), \\
    \id_{\Sph^1 \ltimes_{\eb_3} \R^3} &= (0, 0_{3 \times 1}), \\
    (\theta, x)^{-1} &= (-\theta, - R_{\eb_3}(-\theta) x ),
\end{align*}
where $R_{\eb_3}(\theta) \in \SO(3)$ is the anti-clockwise rotation of an angle $\theta$ about the axis $\eb_3$.
Then $\Sph^1 \ltimes_{\eb_3} \R^3$ is isomorphic to the subgroup
\begin{align*}
    \SE_{\eb_3}(3) := \{
        (R, x) \in \SE(3) \; \vline \; R \eb_3 = \eb_3
    \} \subset \SE(3).
\end{align*}

Define $\alpha : \SE_{\eb_3} \times \vinsT \to \vinsT$ by
\begin{align}
    \alpha(S, (\xi_\imu, b_\imu, T, p_i)) := (S^{-1} P_\imu, R_S^\top v_\imu, b_\imu, T, S^{-1}(p_i)).
    \label{eq:alpha_invariance}
\end{align}
Then $\alpha$ is a right group action of $\SE_{\eb_3}(3)$ on $\vinsT$.
For a given $S \in \SE_{\eb_3}(3)$, the action $\alpha(S, \cdot)$ represents a change of inertial reference frame from $\{0\}$ to $\{1\}$ where $S$ is the pose of $\{1\}$ with respect to $\{0\}$.
Moreover, any change of reference $S \in \SE_{\eb_3}(3)$ leaves the direction of gravity $\eb_3$ unchanged.

\begin{proposition} \label{prop:invariance_action}
The dynamics \eqref{eq:vins_dynamics} and measurements \eqref{eq:measurement_function} of VI-SLAM are invariant with respect to $\alpha$; that is,
\begin{align*}
    f_{(\Omega, a)} (\alpha(S, (\xi_\imu, b_\imu, T, p_i))) &= \td \alpha_S f_{(\Omega, a)} (\xi_\imu, b_\imu, T, p_i), \\
    h(\alpha(S, (\xi_\imu, b_\imu, T, p_i))) &= h(\xi_\imu, b_\imu, T, p_i),
\end{align*}
for any $S \in \SE_{\eb_3}(3)$.
\end{proposition}

\section{Symmetry of VI-SLAM}

The VI-SLAM symmetry action proposed in this paper combines the symmetry for IMU dynamics developed in \cite{2014_barrau_InvariantParticleFiltering} with the VSLAM symmetry developed in \cite{2021_vangoor_ConstructiveObserverDesign}.
Before discussing the full symmetry group, we discuss the separate symmetries and how they lead to lower linearisation error in the filter development given in the sequel.
This section of the paper is written in a more tutorial style to provide the reader with intuition underlying the proposed algorithm.

The advantage of the extended Euclidean symmetry ($\SE_2(3)$) used for the navigation states lies in providing a locally linear coordinate representation of the ideal IMU dynamics.
That is, for ideal IMU dynamics, then using this representation leads to zero linearisation error during the propagation step of the filter, assuming appropriate Gaussian noise models in local coordinates of course.
This property is lost once the bias states and calibration states are added, however, the resulting update still has considerably lower linearisation error in these coordinates than classical formulations.
Section \ref{sec:imu_symmetry} is based on prior work by Barrau \etal \cite{2014_barrau_InvariantParticleFiltering}, and the results presented therein can equally be applied in an IEKF.

The advantage of the scaled orthogonal transform ($\SOT(3)$) symmetry used for the landmark states lies in providing a framework in which the measurement linearisation error can be minimized.
To make this point clear we analyse the common landmark parametrisations used in the literature and study the linearisation error.
This provides a clear theoretical justification for the inverse depth parametrisation that is state-of-the-art in VIO algorithms and goes on to demonstrate that the $\SOT(3)$ symmetry leads to lower measurement linearisation error again.
The material in Section \ref{sec:landmark_symmetry} is novel to this paper.

\subsection{Symmetry of IMU Dynamics}
\label{sec:imu_symmetry}

Let $\inertf : \vecL \to \gothX(\inertM)$ denote the IMU dynamics considered in \eqref{eq:vins_dynamics} without bias; that is,
\begin{align}
    \ddt (R_\imu, x_\imu, v_\imu) &= f_{(\Omega, a)}(R_\imu, x_\imu, v_\imu), \\
    &= (
    R_\imu \Omega^\times,
    v_\imu,
    R_\imu a + g \eb_3,
    ).
    \label{eq:bias_free_imu_dynamics}
\end{align}
Filter designs such as the EKF, on-manifold EKF \cite{2013_hertzberg_IntegratingGenericSensor}, MEKF \cite{2007_mourikis_MultistateConstraintKalman} and EqF \cite{2023_vangoor_EquivariantFilterEqF} model the evolution of the probability distribution of the system state, given an initial distribution.
In each of these filters, the filter's error coordinates are taken to be normally distributed, and the dynamics of these error coordinates are linearised to propagate the estimated distribution.

Naive EKF designs model the robot's attitude through embedded coordinates as a unit quaternion $q \in \mathbb{H} \simeq \R^4$ \cite{2001_marins_ExtendedKalmanFilter,2006_choukroun_NovelQuaternionKalman,2006_sabatini_QuaternionbasedExtendedKalman,2019_anbu_IntegrationInertialNavigation}.
This leads to the (over-parametrised) error coordinates,
\begin{align*}
    \varepsilon_{\text{\tiny EKF}}(\hat{\xi}_\imu, \xi_\imu)
    := \begin{pmatrix}
        q_\imu - \hat{q}_\imu \\
        x_\imu - x_{\hat{\xi_\imu}} \\
        v_\imu - x_{\hat{\xi_\imu}}
    \end{pmatrix}.
\end{align*}
In the MEKF and the typical on-manifold EKF, the error coordinates are instead defined using the logarithm of $\SO(3)$,
\begin{align*}
    \varepsilon_\text{\tiny MEKF}(\hat{\xi_\imu}, \xi_\imu) :=
    \begin{pmatrix}
        \log_{\SO(3)}(R_\imu R_{\hat{\xi_\imu}}^\top)^\vee \\
        x_\imu - x_{\hat{\xi_\imu}} \\
        v_\imu - x_{\hat{\xi_\imu}}
    \end{pmatrix}.
\end{align*}
Note that, unlike the EKF, the MEKF has a minimal (9-dimensional) representation of filter error.

In an EqF design, the filter state is part of a Lie group rather than the state space $\inertM$ \cite{2023_vangoor_EquivariantFilterEqF}.
Consider the Lie group $\SE_2(3)$ with right group action $\varphi : \SE_2(3) \times \inertM \to \inertM$ given by
\begin{align}
    \varphi(A, \xi_\imu) := (R_\imu R_A, x_\imu + R_\imu x_A, v_\imu + R_\imu v_A), \label{eq:inertial_phi_dfn}
\end{align}
where $A = (R_A, x_A, v_A) \in \SE_2(3)$.
In this case, $\SE_2(3)$ and $\inertM$ are isomorphic as smooth manifolds, but the distinction is important for systems on general homogeneous spaces, such as the full VI-SLAM system, where the Lie group may be of a higher dimension than the state space.
Choose the origin configuration $\mr{\xi}_\imu = (I_3, 0, 0) \in \inertM$, and define a local coordinate chart $\vartheta^\imu(\xi_\imu) := \log_{\SE_2(3)}^\vee(R_\imu, x_\imu, v_\imu) \in \R^9$.
Then the EqF error coordinates are
\begin{align}
    & \varepsilon_\text{\tiny EqF}(\hat{A}, \xi_\imu)
    := \vartheta^\text{\tiny I}(\varphi(\hat{A}^{-1}, \xi_\imu)),
    \\
    &= \log_{\SE_2(3)} \begin{pmatrix}
        R_\imu R_{\hat{A}}^\top & x_\imu - R_\imu R_{\hat{A}}^\top x_{\hat{A}} & v_\imu - R_\imu R_{\hat{A}}^\top v_{\hat{A}} \\
        0 & 1 & 0 \\ 0 & 0 & 1
    \end{pmatrix}^\vee \notag,
\end{align}
where $\hat{A} = (R_{\hat{A}}, x_{\hat{A}}, v_{\hat{A}}) \in \SE_2(3)$ is the filter state.

Each of these filters model their error coordinates as being drawn from a normal distribution with zero mean and covariance $\Sigma$, $\varepsilon \sim N(0, \Sigma)$, where $\Sigma$ is the filter's Riccati matrix.
The reported probability distribution of these filters is not guaranteed to match the true distribution of the state in general, as the propagation of the covariance depends on the linearisation of the error dynamics.
For the EqF, however, this linearisation is dependent on the chosen symmetry group $\grpG$, rather than on a chosen set of coordinates as in the EKF.
In some special cases, the system dynamics are \emph{group affine} with respect to the symmetry group $\grpG$, and this results in an exactly linear propagation of the error coordinates \cite{2018_barrau_InvariantKalmanFiltering}.

Barrau and Bonnabel \cite{2014_barrau_InvariantParticleFiltering} showed that the bias-free IMU dynamics are group affine with respect to the action \eqref{eq:inertial_phi_dfn} of $\SE_2(3)$.
As a result, using the Lie group $\SE_2(3)$, the propagation of bias-free IMU dynamics in the EqF has no linearisation error in normal coordinates.
Hence, as long as the initial distribution of the error coordinates is Gaussian, the probability distribution estimated by the EqF will match the true distribution of the state exactly.
This is demonstrated in the following example.

Let the initial value of the true state $\xi_\imu(0) = \exp_{\SE_2(3)}(\eta_0^\wedge)$, where $\eta_0 \in \R^{9}$ is drawn from the distribution $N(0, \Sigma_0)$ with
\begin{align*}
    \Sigma_0 = \begin{pmatrix}
        0.2^2 I_3 & 0 & 0 \\
        0 & 0.01^2 I_3 & 0 \\
        0 & 0 & 0.01^2 I_3
    \end{pmatrix}.
\end{align*}
Initialise each of the filters with this data, and let the angular velocity be $\Omega = (0,0,0.1)$~rad/s and the linear acceleration be $a = (0.1,0.0,0.0)$~m/s$^2$.

\begin{figure}[!htb]
    \centering
    \includegraphics[width=0.5\linewidth]{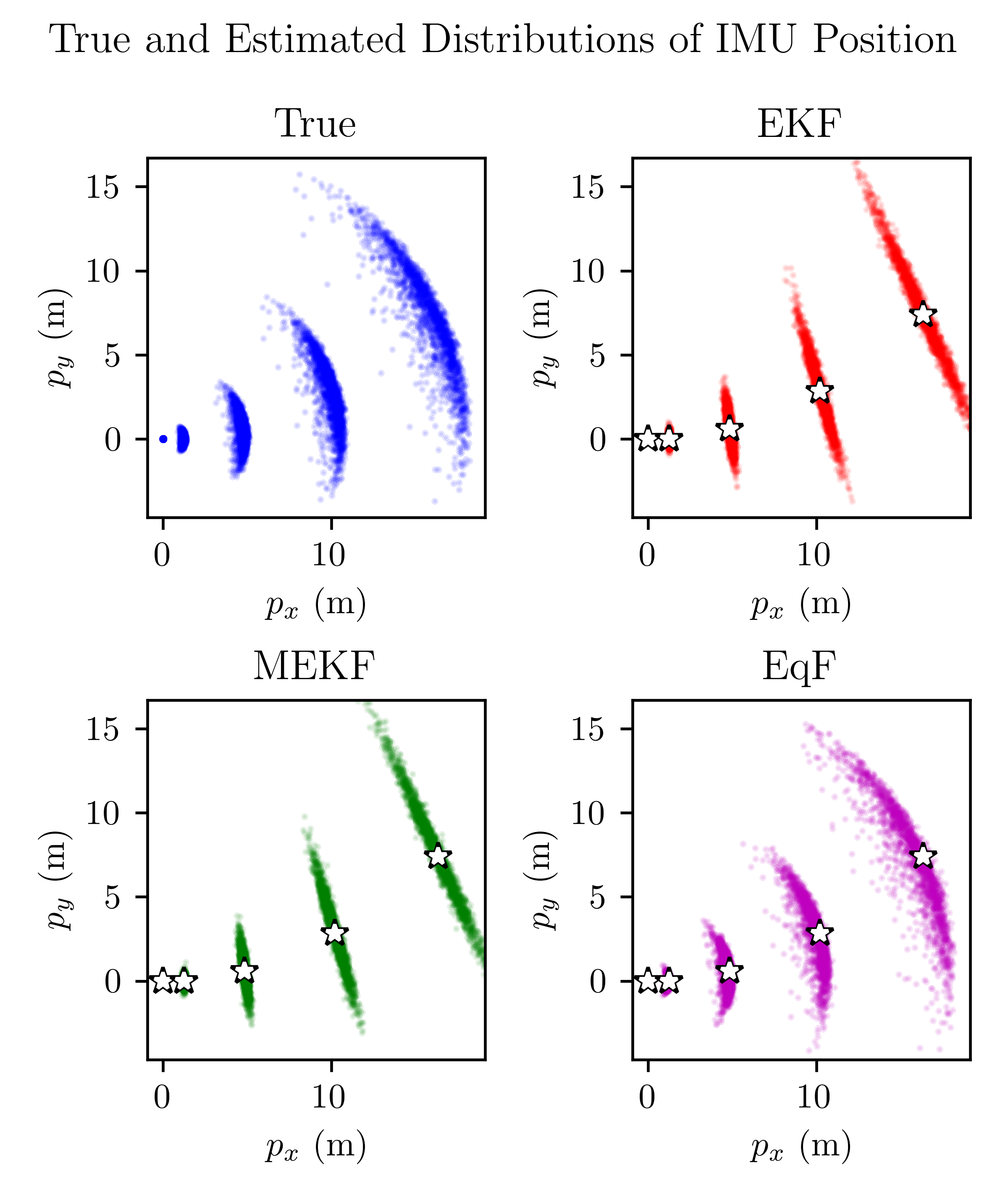}
    \caption{The true distribution of IMU positions at increments of 5~s obtained from integrating the dynamics \eqref{eq:bias_free_imu_dynamics} compared with the estimated distributions from the EKF, MEKF, and EqF.}
    \label{fig:se23_distribution}
\end{figure}

Figure \ref{fig:se23_distribution} shows the true and estimated distributions of the IMU position after integrating in 5~s increments.
At each time, the true distribution of the system is shown by sampling 2000 particles according to the initial distribution of the state and integrating them independently.
The estimated distributions are obtained by integrating the filter equations in 5~s increments, before sampling 2000 points in the error coordinates and mapping them to the estimated state using the filter's observer state.
The figure clearly shows the advantages of applying an appropriate symmetry to the propagation of IMU uncertainty, as the distribution reported by the EqF matches the true distribution far more closely than that of the EKF or MEKF.

\begin{remark}
The inclusion of biases in the gyroscope and accelerometer measurements means that the full IMU dynamics \eqref{eq:vins_dynamics} are not group affine with respect to the action \eqref{eq:inertial_phi_dfn} \cite{2018_brossard_InvariantKalmanFiltering}.
However, the linearisation error introduced by the bias states is proportional to the error in the bias estimates, which is often small and can be reduced quickly by using an initialisation maneuver in practice.
Recent work by Fornasier \etal \cite{2022_fornasier_OvercomingBiasEquivariant,2022_fornasier_EquivariantFilterDesign} shows that alternative symmetries exist for including input biases in an EqF, but their analysis is beyond the scope of the current paper and left for future work.
\end{remark}

An EKF is not an optimal estimator, unlike the linear Kalman filter, due to the accumulation of linearisation errors over time.
By exploiting symmetry properties as above, the linearisation error in each step of the EKF can be reduced or even eliminated completely.
This leads to improved filter designs that closely reflect the stochastics of the underlying system and provide more accurate state estimates.

\subsection{Symmetry of Visual Landmarks}
\label{sec:landmark_symmetry}

Consider the simplified system of a single landmark $q \in \R^3 \setminus \{ 0 \}$ in the camera-fixed frame.
If the camera-fixed angular and linear velocity are $\Omega_C, v_C \in \R^3$, respectively, then the dynamics of the landmark are
\begin{align}
    \dot{q} = f^\text{\tiny V}_{(\Omega_C, v_C)}(q) := -\Omega_C \times q - v_C.
    \label{eq:visual_dynamics_landmark}
\end{align}
The visual measurement of the landmark is
\begin{align}
    h^\vis(q) = \frac{q}{\vert q \vert}
    \label{eq:visual_measurement_landmark}
\end{align}

A parametrisation of the landmark is a diffeomorphism $\varsigma: U \subset \R^3 \to \visualM \subset \R^k$, where $U$ is an open subset of $\R^3$ and $\visualM$ is a smooth 3-dimensional submanifold of $\R^k$ for some $k \geq 3$.

The \emph{Euclidean} parametrisation was commonly used in early works on visual SLAM \cite{2006_bailey_ConsistencyEKFSLAMAlgorithm} and is defined by
\begin{align}
    \varsigma_\text{Euc}(q) &:= q, &
    \varsigma_\text{Euc}^{-1}(z) &:= z.
    \label{eq:visual_param_euclid}
\end{align}
The \emph{inverse-depth} parametrisation, and variants thereof, are used more frequently in recent literature \cite{2008_civera_InverseDepthParametrization,2017_bloesch_IteratedExtendedKalman,2020_geneva_OpenVINSResearchPlatform}.
Here, we consider the archetypical example given by
\begin{align}
    \varsigma_\text{ID}(q) &:= \begin{pmatrix}
        \arccos(\frac{q_1}{\vert q \vert}) \\
        \mathrm{atan2}(q_2, q_3) \\
        \frac{1}{\vert q \vert}
    \end{pmatrix},
    &
    \varsigma_\text{ID}^{-1}(z) &:= \frac{1}{z_3} \begin{pmatrix}
        \cos(z_1) \\
        \sin(z_2) \sin(z_1) \\
        \cos(z_2) \sin(z_1)
    \end{pmatrix}
    \label{eq:visual_param_invdepth}
\end{align}

We introduce a new parametrisation for landmarks with visual measurements by exploiting the $\SOT(3)$ symmetry actions developed in \cite{2021_vangoor_ConstructiveObserverDesign}.
The \emph{polar} parametrisation is a novel parametrisation for visual landmarks, defined by
\begin{align}
    \varsigma_{\SOT(3)}(q) &:=
        \begin{pmatrix}
            \arccos(\frac{q_3}{\vert q \vert}) \frac{q_2}{\vert \eb_3 \times q \vert} \\
            \arccos(\frac{q_3}{\vert q \vert}) \frac{- q_1}{\vert \eb_3 \times q \vert} \\
            -\log(\vert q \vert)
        \end{pmatrix}, \notag \\
    \varsigma_{\SOT(3)}^{-1}(z) &:=
        \exp_{\SOT(3)}\left(
            (\begin{pmatrix}
            z_1 & z_2 & 0 & z_3
        \end{pmatrix}^\top)_{\sot(3)}^\wedge
        \right)^{-1} \eb_3.
    \label{eq:visual_param_sot3}
\end{align}
This parametrisation provides normal coordinates for $\R^3$ about $\eb_3$ with respect to the right action of $\SOT(3)$ defined in Lemma \ref{lem:simple_landmark_equivariance}.

\begin{lemma}\label{lem:simple_landmark_equivariance}
    Let $\varphi^\vis : \SOT(3) \times \R^3 \setminus \{ 0 \} \to \R^3 \setminus \{ 0 \}$ and $\rho^\vis : \SOT(3) \times \Sph^2 \to \Sph^2$ be defined by
    \begin{align}
        \varphi^\vis(Q, q) &:= c_Q^{-1} R_Q^\top q, \notag \\
        \rho^\vis(Q, y) &:= R_Q^\top y.
        \label{eq:visual_group_actions}
    \end{align}
    Then $\varphi^\vis$ and $\rho^\vis$ are transitive right group actions, and the visual measurement function \eqref{eq:visual_measurement_landmark} is equivariant with respect to these actions; that is,
    \begin{align*}
        h^\vis  (\varphi^\vis(Q, q)) = \rho^\vis ( Q, h^\vis(q)),
    \end{align*}
    for all $Q \in \SOT(3)$ and $q \in \R^3 \setminus \{0\}$.
\end{lemma}

Let $\hat{q} \in \R^3 \setminus \{0\}$ be the estimated landmark position and $q \in \R^3 \setminus \{0\}$ be the true landmark position.
Then the \emph{parametrised state error} is
\begin{align*}
    \varepsilon := \varsigma(q) - \varsigma(\hat{q}),
\end{align*}
for a particular parametrisation $\varsigma$.
The true state can be identified in terms of the estimated state and the parametrised state error,
\begin{align*}
    q = \varsigma^{-1}(\varepsilon + \varsigma(\hat{q})).
\end{align*}
Linearising the dynamics \eqref{eq:visual_dynamics_landmark} and measurement \eqref{eq:visual_measurement_landmark} in terms of $\varepsilon$ about $\varepsilon=0$ yields
\begin{align}
    f^\text{\tiny V}_{(\Omega, v)}(q) &= f^\text{\tiny V}_{(\Omega, v)} \circ \varsigma^{-1}(\varepsilon + \varsigma(\hat{q}))
    \label{eq:state_lin_error} \\
    &= f^\text{\tiny V}_{(\Omega, v)}(\hat{q})
    + \tD_q |_{\hat{q}} f^\text{\tiny V}_{(\Omega, v)}(q) \cdot \tD_s |_{\varsigma(\hat{q})} \varsigma(s) [\varepsilon]
    + \mu^f(q), \notag \\
    h^\text{\tiny V}(q) &= h^\text{\tiny V} \circ \varsigma^{-1}(\varepsilon + \varsigma(\hat{q}))
    \label{eq:output_lin_error} \\
    &= h^\text{\tiny V}(\hat{q})
    + \tD_p |_{\hat{q}} h^\text{\tiny V}(q) \cdot \tD_s |_{\varsigma(\hat{q})} \varsigma(s) [\varepsilon]
    + \mu^h(q), \notag
\end{align}
where $\mu^f(q)$ and $\mu^h(q)$ are the the dynamics and measurement linearisation errors, respectively, and capture all higher order terms.
In general, both $\mu^f(q)$ and $\mu^h(q)$ are $O(\vert \varepsilon \vert^2)$, but there are important exceptions.
First, the dynamics \eqref{eq:visual_dynamics_landmark} are exactly linear in the Euclidean parametrisation \eqref{eq:visual_param_euclid}, and hence $\mu^f_\text{Euc}(q) \equiv 0$.
Second, the equivariant output approximation proposed in \cite{2023_vangoor_EquivariantFilterEqF} is available when using the polar parametrisation as these are the normal coordinates associated with the action $\varphi^\vis$ \eqref{eq:visual_group_actions}.
Specifically, applying \cite[Lemma V.3]{2023_vangoor_EquivariantFilterEqF} to this example yields
\begin{align*}
    C^\star_t \varepsilon
    &= \frac{1}{2} (\tD_E |_\id \rho^\text{\tiny V}(E, y) + \tD_E |_\id \rho^\text{\tiny V}(E, \eb_3)) \varepsilon^\wedge_{\sot(3)}, \\
    &= \frac{1}{2} \ddt |_{t=0} (\rho^\text{\tiny V}(e^{t \varepsilon^\wedge_{\sot(3)}}, y) + \rho^\text{\tiny V}(e^{t \varepsilon^\wedge_{\sot(3)}}, \eb_3)), \\
    &= \frac{1}{2} \ddt |_{t=0} (-R_{e^{t \varepsilon^\wedge_{\sot(3)}}} y -R_{e^{t \varepsilon^\wedge_{\sot(3)}}} \eb_3), \\
    &= \frac{1}{2} (-\Omega^\times_{\varepsilon} y -\Omega^\times_{\varepsilon} \eb_3), \\
    &= \frac{1}{2} (y + \eb_3)^\times \Omega_{\varepsilon},
\end{align*}
and therefore,
\begin{align*}
    C_t^\star =
        \frac{1}{2} (y + \eb_3)^\times
    \begin{pmatrix}
        1 & 0 & 0 \\
        0 & 1 & 0 \\
        0 & 0 & 0
    \end{pmatrix}
\end{align*}
This provides an alternative measurement approximation that reduces the output linearisation error to $\mu^h_{\SOT(3)}(q) = O(\vert \varepsilon \vert^3)$ by exploiting equivariance and using the available system measurement value $y = h(q)$.



To see the effects of different parametrisations on the linearisation errors, let $\hat{q} = (0,0,5)$ and consider the domain
\begin{align}
    U = \{ (z \tan(\theta), 0, z) \in \R^3 \mid z \in (0.1, 10), \; \theta \in (-\frac{\pi}{6}, \frac{\pi}{6}) \}.
    \label{eq:linearisation_domain_U}
\end{align}

\begin{figure*}[!tb]
    \centering
    \includegraphics[width=0.9\linewidth]{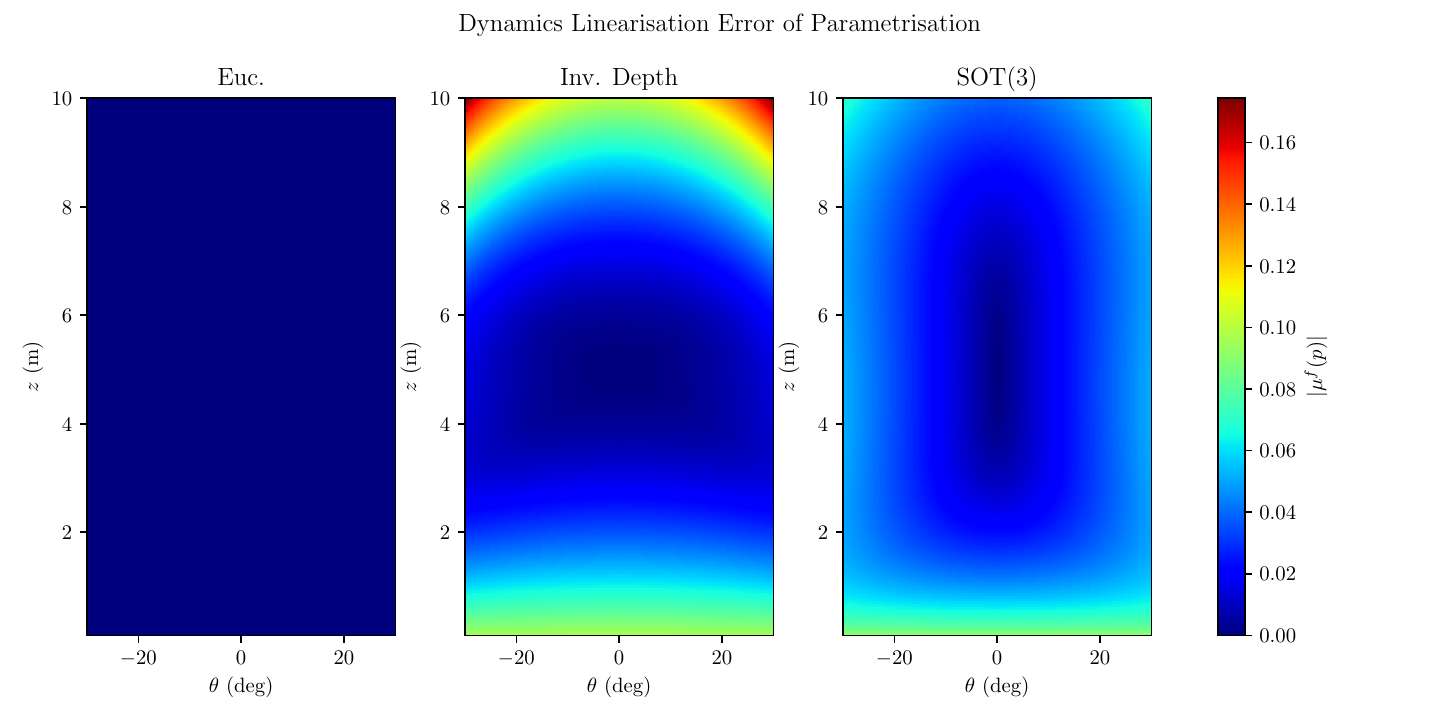}
    \caption{The norm of linearisation error \eqref{eq:state_lin_error} of the landmark dynamics \eqref{eq:visual_dynamics_landmark} for the Euclidean \eqref{eq:visual_param_euclid}, inverse-depth \eqref{eq:visual_param_invdepth}, and polar \eqref{eq:visual_param_sot3} parametrisations over the domain defined in \eqref{eq:linearisation_domain_U}.
    The Euclidean parametrisation has zero dynamics linearisation error since the dynamics are exactly linear in these coordinates.
    }
    \label{fig:dynamics_linearisation}
\end{figure*}

Let $\Omega_C = (0.0,0.2,0.0)$ and $v_C = (0.0,0.0,0.1)$.
This linear and angular velocity is typical of a camera moving forward and turning about the vertical axis of the image.

Figures \ref{fig:dynamics_linearisation} and \ref{fig:output_linearisation} show the dynamics and output linearisation errors $\mu^f(q)$ and $\mu^h(q)$, respectively, for each of the parametrisation (\ref{eq:visual_param_euclid},\ref{eq:visual_param_invdepth},\ref{eq:visual_param_sot3})
Figure \ref{fig:output_linearisation} also shows the equivariant output approximation \cite{2023_vangoor_EquivariantFilterEqF} in the polar parametrisation.
As expected, the Euclidean parametrisation yields no dynamics linearisation error over any part of the domain.
Meanwhile, the inverse-depth and polar parametrisations have non-zero dynamics linearisation errors that increase toward the edges of the domain, with the polar parametrisation having a lower maximum linearisation error.
On the other hand, for the output linearisation there is a clear advantage to using the inverse-depth or polar over the Euclidean parametrisation.
Finally, while the ordinary performance of the inverse-depth and polar parametrisations is similar, the polar parametrisation is able to use the equivariant output approximation which greatly reduces its output linearisation error everywhere in $U$.

\begin{figure}[!htb]
    \centering
    \includegraphics[width=0.6\linewidth]{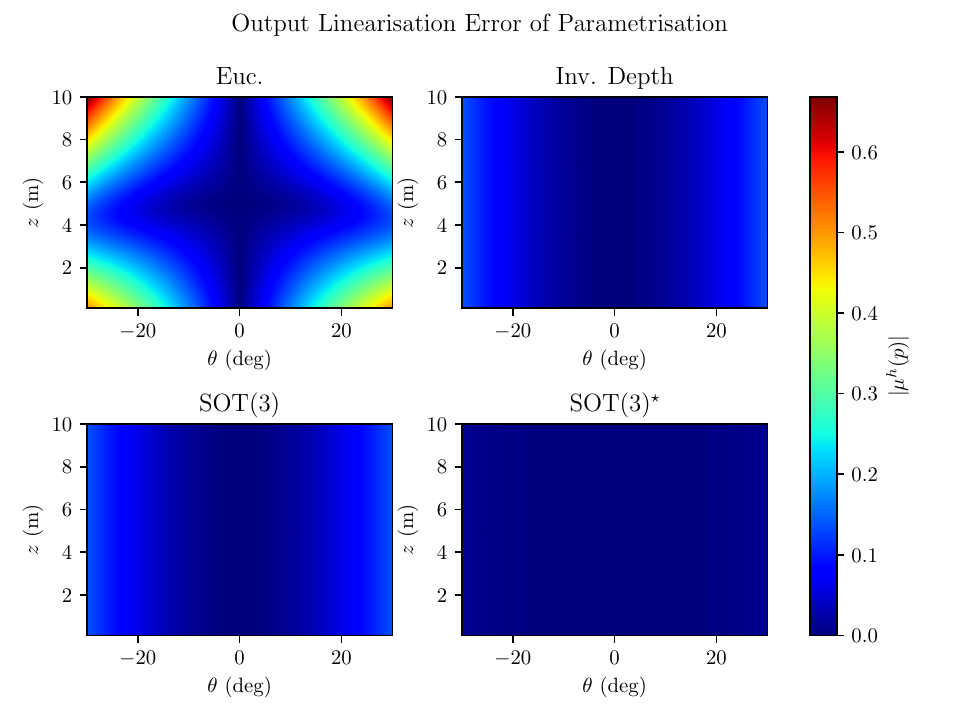}
    \caption{The norm of linearisation error \eqref{eq:output_lin_error} of the visual landmark measurement \eqref{eq:visual_measurement_landmark} for the Euclidean \eqref{eq:visual_param_euclid}, inverse-depth \eqref{eq:visual_param_invdepth}, and polar \eqref{eq:visual_param_sot3} parametrisations over the domain defined in \eqref{eq:linearisation_domain_U}.
    The bottom-right subplot shows the linearisation error obtained when applying the equivariant output approximation \cite{2021_vangoor_ConstructiveObserverDesign} in the polar parametrisation.
    }
    \label{fig:output_linearisation}
\end{figure}

\subsection{Composite Symmetry for VI-SLAM}

Define the \emph{VI-SLAM Group} to be
\begin{align}
    \vinsG := \SE_2(3) \times \SOT(3)^n.
    \label{eq:vins_group_dfn}
\end{align}
This is a product Lie group in the sense of \eqref{eq:product_lie_group}.
This group captures the symmetries that are fundamental to the VIO system: $\SE_2(3)$ for the IMU dynamics and $\SOT(3)^n$ for $n$ landmarks with visual measurements.
In a practical implementation, however, it is also necessary to consider the extrinsic calibration parameters and IMU biases, for which an $\SE(3)$ and an $\R^6$ symmetry can be used, respectively.
To this end, define the extended direct product group
\begin{align}
    \grpG &:= \SE_2(3) \times \R^6 \times \SE(3) \times \SOT(3)^n, \notag \\
    &\simeq \vinsG \times \R^6 \times \SE(3).
    \label{eq:full_vins_group_dfn}
\end{align}
We denote a typical element as $X = (A, \beta, B, Q_1, ..., Q_n) \in \grpG$, and frequently use the shorthand $(A, \beta, B, Q_i)$.

\begin{lemma} \label{lem:state_action}
    The map $\phi : \grpG \times \vinsT \to \vinsT$ defined by
    \begin{align}
        \phi & ((A, \beta, B, Q_i), (\xi_\imu, b_\imu, T, p_i)) \notag \\
        &:= (\varphi^\imu(A, \xi_\imu), b_\imu + \beta,
        P_A^{-1} T B, P_\imu T B Q_i^{-1} T^{-1} P_\imu^{-1} (p_i)), \label{eq:total_state_action}
    \end{align}
    where $\varphi^\imu$ is defined as in \eqref{eq:inertial_phi_dfn},
    is a transitive right group action.
\end{lemma}

\begin{lemma}\label{lem:compatible_phi_alpha}
The group action $\phi$ \eqref{eq:total_state_action} is compatible with the VI-SLAM invariance $\alpha$ \eqref{eq:alpha_invariance}; that is,
\begin{align*}
    \phi(X, \alpha(S, \xi))
    = \alpha(S, \phi(X, \xi)),
\end{align*}
for all $\xi \in \vinsT$, $S \in \SE_{\eb_3}(3)$ ,and $X \in \grpG$.
\end{lemma}

\begin{lemma}\label{lem:output_action}
The map $\rho : \grpG \times \visualN \to \visualN$ defined by
\begin{align}
    \rho((A, \beta, B, Q_i), (y_i)) := (R_{Q_i}^\top y_i),
\end{align}
is a right group action.
Additionally, the measurement function \eqref{eq:measurement_function} is equivariant with respect to the actions $\phi$ \eqref{eq:total_state_action} and $\rho$; that is,
\begin{align*}
    h(\phi(X, \xi)) = \rho(X, h(\xi)),
\end{align*}
for all $X \in \grpG$ and $\xi \in \vinsT$.
\end{lemma}

Overall, the structure presented provides a complete symmetry of the VIO dynamics and measurement.
This is summarised by the diagram
\begin{equation*}
\begin{gathered}
   \xymatrix{
       \vinsT \ar[r]^{\alpha_S} \ar[d]_{\phi_X} \ar@/^1.0pc/[rr]^{h} &
       \vinsT \ar[r]^h   \ar[d]_{\phi_X} &
       \visualN            \ar[d]_{\rho_X} \\
       \vinsT \ar[r]^{\alpha_S} \ar@/_1pc/[rr]_{h} &
       \vinsT \ar[r]^h &
       \visualN
   }
\end{gathered}
\end{equation*}
which commutes for any $X \in \grpG$ and $S \in \SE_{\eb_3}(3)$.

The proposed EqF formulation has two key differences to our previous work \cite{2021_vangoor_EquivariantFilterVisual}.
First, while both the proposed EqF and the EqF presented in \cite{2021_vangoor_EquivariantFilterVisual} feature the same $\SE_2(3)$ symmetry associated with the inertial states, the EqF presented in \cite{2021_vangoor_EquivariantFilterVisual} did not include the translation and yaw components directly.
Instead, the translation and yaw components (corresponding to the invariance $\alpha_S$) were excluded from the filter and a secondary \emph{bundle lift} was used to compute an update to the these states separately from the EqF equations.
The advantage of the new formulation is that the EqF now provides a covariance estimate for the translation and yaw.
The second key difference is the inclusion of the camera extrinsics in the VI-SLAM state with an associated $\SE(3)$ symmetry.
In \cite{2021_vangoor_EquivariantFilterVisual}, the camera extrinsics were simply assumed to be fixed constant.

\section{Equivariant Filter for VIO}

While the performance of an EKF for VIO depends largely on what coordinates are chosen, the performance of an EqF depends instead on the symmetry used.
The VI-SLAM symmetry yields a few key advantages.

The EqF equipped with the VI-SLAM symmetry is a consistent estimator for VIO.
In Lemma \ref{lem:lift}, Lemma \ref{lem:state_action} is used to lift the system dynamics from the state space manifold to the chosen symmetry group $\grpG$, and the resulting lift $\Lambda$ is shown to be invariant to the reference frame transformation $\alpha$ \eqref{eq:alpha_invariance}.
Then, since both the group action $\phi$ \eqref{eq:total_state_action} and the lift $\Lambda$ \eqref{eq:lift_dfn} are compatible with $\alpha$ \eqref{eq:alpha_invariance} (c.f. Lemmas \ref{lem:compatible_phi_alpha}, \ref{lem:lift}), the proposed EqF naturally also respects the invariance.
Specifically, if the Riccati matrix $\Sigma$ of the EqF is treated as a covariance, then the Fisher information $\Sigma^{-1}$ is non-increasing along the unobservable directions spanned by the differential of $\alpha$, and therefore the EqF is consistent by \cite[Theorem 2]{2019_brossard_ExploitingSymmetriesDesign}.

The EqF uses a higher order (more precise) output approximation than that available to the EKF.
The system output function is shown to be equivariant in Lemma \ref{lem:output_action}, and the local coordinates $\vartheta$ chosen in Section \ref{sec:local_coords} are normal with respect to the group action $\phi$ \eqref{eq:total_state_action}.
Therefore, the equivariant output approximation $C_t^\star$ is available \cite[Lemma V.3]{2023_vangoor_EquivariantFilterEqF}, and, unlike the usual first-order approximation of the output function provided by an EKF, it has no second order error terms.
That is,
\begin{align*}
    y - h(\hat{\xi}) =  C_t^\star \varepsilon + O(\vert \varepsilon \vert^3),
\end{align*}
where $y$ is the system measurement, $\hat{\xi}$ is the EqF state estimate, and $\varepsilon$ is the local error coordinates of the EqF.

\subsection{Lifted Dynamics and Consistency}

The existence of a transitive action by the VI-SLAM group on the VI-SLAM manifold guarantees the existence of a lift for the system dynamics \eqref{eq:vins_dynamics} in the sense of \cite{2020_mahony_EquivariantSystemsTheory}.
Although \cite{2020_mahony_EquivariantSystemsTheory} provides a constructive algorithm to build lifts, in practice it is easiest to guess the lift and then prove it satisfies the required conditions.

\begin{lemma}\label{lem:lift}
    The map $\Lambda : \vinsT \times (\R^3 \times \R^3) \to \gothg$, given by
    \begin{align}
        \label{eq:lift_dfn}
        \Lambda & ((\xi_\imu, b, T, p_i), (\Omega, a)) \\
        &:= \left(
            (\Omega, R_\imu^\top v, a)^\wedge_{\se_2(3)}, \;
            0 , \;
            (\Omega_C, v_C)^\wedge_{\se(3)},
        \right. \notag \\ &\hspace{1cm} \left.
            \left( \Omega_C + \frac{q_i^\times v_C}{\Vert q_i \Vert^2}, \frac{q_i^\top v_C}{\Vert q_i \Vert^2} \right)_i
        \right), \notag \\
        q_i &:= (P T)^{-1} (p_i),
        \quad
        (\Omega_C, v_C)^\wedge_{\se(3)} := \Ad_{T^{-1}} (\Omega, R_\imu^\top v)^\wedge_{\se(3)}, \notag
    \end{align}
    is a lift \cite{2020_mahony_EquivariantSystemsTheory} of the system dynamics \eqref{eq:vins_dynamics}.
    That is,
    \begin{align}
        \tD_{E} |_\id \phi_{(\xi_\imu, b, T, p_i)} (E) \Lambda((\xi_\imu, b, T, p_i), (\Omega, a)) = f_{(\Omega, a)} (\xi_\imu, b, T, p_i)
        \label{eq:lift_condition}
    \end{align}
    Moreover, $\Lambda$ is invariant with respect to the action $\alpha$ \eqref{eq:alpha_invariance}; i.e.
    \begin{align*}
        \Lambda(\alpha(S, (\xi_\imu, b, T, p_i)), (\Omega, a)) = \Lambda((\xi_\imu, b, T, p_i), (\Omega, a)),
    \end{align*}
    for all $S \in \SE_{\eb_3}(3)$.
\end{lemma}

\begin{proof}
The proof that $\Lambda$ satisfies the lift condition \eqref{eq:lift_condition} closely follows the proof of \cite[Lemma 4.4]{2021_vangoor_ConstructiveObserverDesign}, and the invariance of $\Lambda$ to $\alpha$ is straightforward.
Both have been omitted to save space.
\end{proof}

\subsection{Origin Choice and Local Coordinates}
\label{sec:local_coords}

Let $\mr{\xi}  = (\mr{\xi}_\imu, \mr{b}_\imu, \mr{T}, \mr{p}_i) \in \vinsT$ denote the chosen, fixed \emph{origin configuration}, with $\mr{p}_i := \mr{P}_\imu \mr{T} \eb_3$ for every $i$.
For generality, we leave the remaining terms $\mr{\xi}, \mr{b}_\imu, \mr{T}$ arbitrary.

Define the map $\vartheta : \calU_{\mr{\xi}} \subset \vinsT \to \R^{21+3n}$ by
\begin{align} \label{eq:state_local_coords}
    &\vartheta(\xi_\imu, b, T, p_i) \\
    &:= \begin{pmatrix}
        \log_{\SE_2(3)}(\mr{R}_\imu^\top R_\imu, \mr{R}_\imu^\top(x_\imu - \mr{x}_\imu), \mr{R}_\imu^\top(v_\imu - \mr{v}_\imu))^\vee \\[0.5em]
        b_\imu - \mr{b}_{\imu} \\[0.5em]
        \log_{\SE(3)}((\mr{P}_\imu \mr{T})^{-1} (P_{\imu} T))^\vee \\[0.5em]
        \varsigma_{\SOT(3)}((P_\imu T_\imu)^{-1} (p_1))\\
        \vdots \\[0.5em]
        \varsigma_{\SOT(3)}((P_\imu T_\imu)^{-1} (p_n))
    \end{pmatrix}, \notag
\end{align}
to be the coordinate chart for $\vinsT$ about $\mr{\xi}$, where $\calU_{\mr{\xi}}$ is a large neighbourhood of $\mr{\xi}$ and $\varsigma_{\SOT(3)}$ is the polar parametrisation \eqref{eq:visual_param_sot3}.
Then $\vartheta$ provides normal coordinates for $\vinsT$ about $\mr{\xi}$ with respect to the action $\phi$.


\subsection{EqF Dynamics}

Denote the true state of the system $\xi = (\xi_\imu, b, T, p_i) \in \vinsT$,
let the input signals be $(\Omega, a) \in \R^3 \times \R^3$,
and denote the measurements as $y = h(\xi)$.

Let $\hat{X} = (\hat{A}, \hat{\beta}, \hat{B}, \hat{Q}_i) \in \grpG$ be the observer state and $\Sigma \in \Sym_+(21+3n)$ be the Riccati matrix, where $\Sym_+(k)$ denotes the set of positive definite $k \times k$ matrices.
Define $\mr{A}_t, B_t, C_t^\star$ to be the state, input, and equivariant output matrices of the EqF as defined in \cite{2023_vangoor_EquivariantFilterEqF}, respectively, and let $\tD_E |_\id \phi_{\mr{\xi}}(E)^\dag$ be a fixed right-inverse of $\tD_E |_\id \phi_{\mr{\xi}}(E)$.
Then the EqF dynamics \cite{2023_vangoor_EquivariantFilterEqF} are defined to be
\begin{align}
    \label{eq:eqf_equations}
    \dot{\hat{X}} &= \hat{X} \Lambda( \phi(\hat{X}, \mr{\xi}), (\Omega, a)) - \Delta \hat{X}, \\
    \Delta &= \tD_E |_\id \phi_{\mr{\xi}}(E)^\dagger \cdot \tD_\xi |_{\mr{\xi}} \vartheta(\xi)^{-1} \Sigma {C_t^\star}^\top N_t^{-1} (y - h(\hat{\xi})), \notag \\
    \dot{\Sigma} &= \mr{A}_t \Sigma  + \Sigma \mr{A}_t^\top + M_\varepsilon +  B_t M_t^m B_t^\top - \Sigma {C_t^\star}^\top N_t^{-1} C^\star_t \Sigma, \notag
\end{align}
where $\Delta \in \gothg$ is the EqF correction term, and
\begin{itemize}
    \item $\Sigma(0) = \Sigma_0 \in \Sym_+(21+3n)$ is the initial Riccati gain,
    \item $M_\varepsilon \in \Sym_+(21+3n)$ is the state gain,
    \item $M_t^m \in \Sym_+(6)$ is the input gain,
    \item $N_t \in \Sym_+(2n)$ is the output gain,
\end{itemize}
following their descriptions in \cite{2023_vangoor_EquivariantFilterEqF}.

At any time, the EqF state estimate is given by
\begin{align}
    \hat{\xi} = (\hat{\xi}_\imu, \hat{b}, \hat{T}, \hat{p}_i) := \phi(\hat{X}, \mr{\xi}). \label{eq:estimated-state}
\end{align}
The EqF is designed around the definition of a global error state, $e = \phi(\hat{X}^{-1}, \xi)$.
The filter is developed by expressing the dynamics and measurements in terms of $e$ and linearising through the local coordinates $\vartheta$.
Figure \ref{fig:eqf_overview} shows how the true state $\xi$, estimated state $\hat{\xi}$, observer state $\hat{X}$, and global error $e$ are related.
From the probabilistic point of view, the Riccati matrix $\Sigma$ may then be thought of as the covariance of the linearised error $\varepsilon = \vartheta(e)$; that is, the EqF estimates the distribution of the true state $\xi$ to satisfy
\begin{align*}
    \vartheta(\phi(\hat{X}^{-1}, \xi)) \sim N(0, \Sigma).
\end{align*}
For a formal discussion of the derivation and probabilistic interpretation of the EqF, we refer the reader to \cite{2023_vangoor_EquivariantFilterEqF,2022_ge_EquivariantFilterDesign}.

\begin{figure}[htb]
    \centering
    \includegraphics[width=0.5\linewidth]{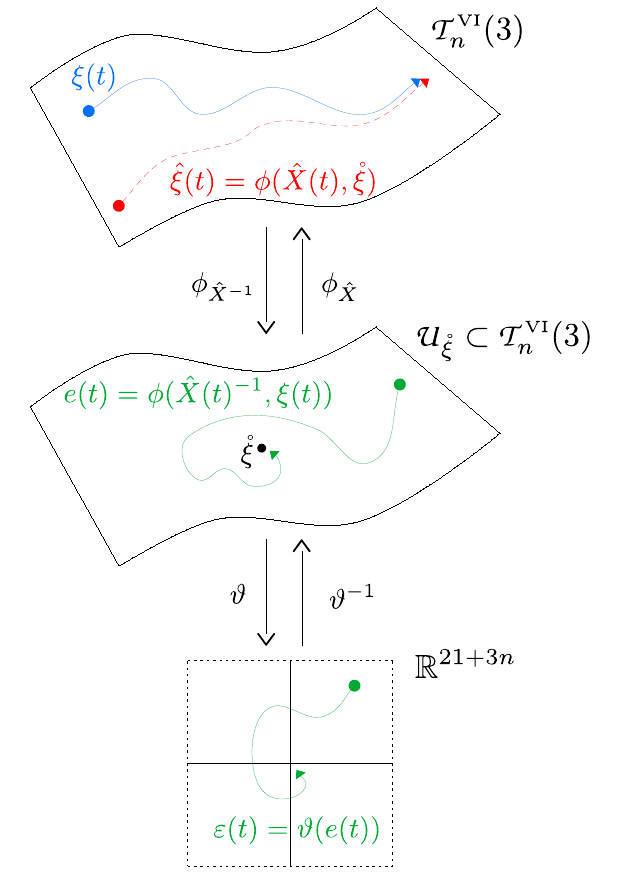}
    \caption{An illustration of the relationships between the true state $\xi$, estimated state $\hat{\xi}$, observer state $\hat{X}$, and global error $e$.
    The true and estimated state are related to the error and the origin by the transformation $\phi_{\hat{X}}$.
    The error $e$ is linearised through local coordinates to yield $\varepsilon = \vartheta(e)$.}
    \label{fig:eqf_overview}
\end{figure}

\section{Experimental Results}

\begin{figure*}[t]
    \centering
    \includegraphics[width=1.0\textwidth]{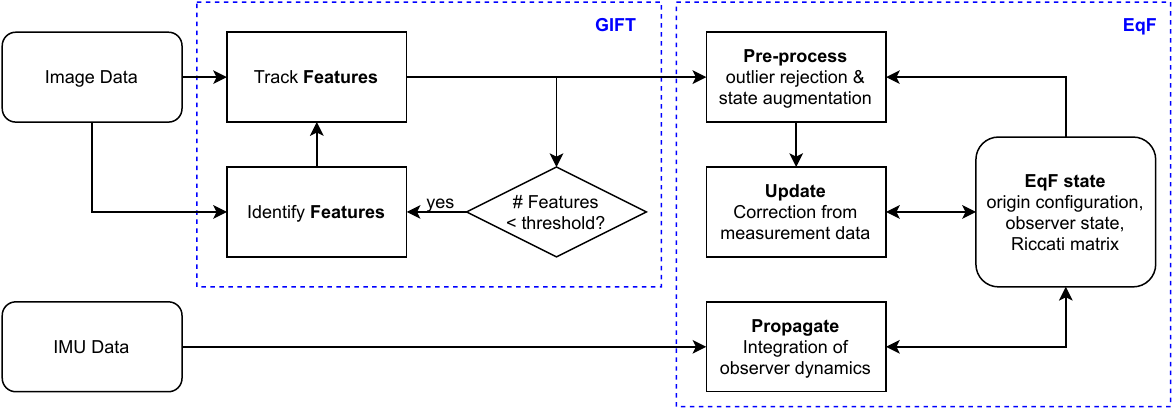}
    \caption{An overview of EqVIO as a system. The key components can be split into the front-end (GIFT) and the back-end (EqF).}
    \label{fig:eqvio-diagram}
\end{figure*}

The EqF equations \eqref{eq:eqf_equations} are discretised and implemented in c++17 using the Eigen matrix library \cite{2010_guennebaud_EigenV3}.
Visual measurements of landmarks are obtained using GIFT \cite{2021_vangoor_ConstructiveObserverDesign, 2021_vangoor_GeneralInvariantFeature} to identify and track image features.
We refer to the resulting system as \emph{EqVIO}.
The four key steps in EqVIO are
\begin{itemize}
    \item \textbf{Features:} The number of features tracked is kept to a fixed limit that can be set by the user. When the number of features that are successfully being tracked falls below a chosen threshold, new image features are identified.
    \item \textbf{Preprocessing:} If a feature is unable to be tracked or is identified as an outlier, the corresponding landmark is removed from the state. If any new features have been identified, they are added to the state with a constant initial covariance.
    \item \textbf{Propagation:} Upon receiving an IMU signal, the EqF state and Ricatti equation are integrated without the correction terms.
    \item \textbf{Update:} After augmenting the state and removing outliers, the EqF state and Ricatti equation are integrated exclusively with the correction terms.
\end{itemize}

Figure \ref{fig:eqvio-diagram} provides an overview of the system components and how each of the steps are linked.

\subsection{Performance Comparisons}

We compared the performance of EqVIO to other VIO systems on two popular public datasets: EuRoC \cite{2016_burri_EuRoCMicroAerial} and UZH FPV \cite{2019_delmerico_AreWeReady}.
Table \ref{tab:algorithms} lists the algorithms considered along with information about the back-end and front-end systems used, including GFFT \cite{1994_shi_GoodFeaturesTrack}, LKT \cite{2004_baker_LucasKanade20Years}, FAST \cite{2006_rosten_MachineLearningHighSpeed}, ORB \cite{2011_rublee_ORBEfficientAlternative}, and BRISK \cite{2011_leutenegger_BRISKBinaryRobust}.

\begin{table}
\caption{Popular VIO algorithms used to compare with the performance of EqVIO.}
\label{tab:algorithms}
\centering
\begin{tabular}{c|c|c|c}
    Algorithm & Ref. & back-end & front-end \\
    \hline
    EqVIO & * & EqF & GFTT + LKT \\
    ROVIO & \cite{2017_bloesch_IteratedExtendedKalman} & iterated EKF & image patches \\
    OpenVINS & \cite{2020_geneva_OpenVINSResearchPlatform} & MSCKF + EKF & FAST + LKT \\
    VINS-mono & \cite{2018_qin_VinsmonoRobustVersatile} & sliding-window & GFTT + LKT + ORB \\
    MSCKF & \cite{2007_mourikis_MultistateConstraintKalman} & MSCKF & FAST + LKT \\
    OKVIS & \cite{2015_leutenegger_KeyframebasedVisualInertial} & sliding-window & GFTT + BRISK \\
\end{tabular}
\end{table}

Where stated, the algorithm performance results have been obtained from the recent benchmark study by Delmerico and Scaramuzza \cite{2018_delmerico_BenchmarkComparisonMonocular}.
Otherwise, the algorithms were compiled using the suggested default configuration and run on an Ubuntu 20.04 desktop computer equipped with an AMD Ryzen 7 3700X 8-Core processor and 16~GB memory.
The tuning parameters of the algorithms were changed between the EuRoC and UZH FPV datasets, but kept constant across all sequences in those datasets. Additional system capabilities are including loop closure and map reuse were disabled where relevant to ensure a fair comparison between algorithms.
The key sensor characteristics and tuning parameters used for EqVIO are listed in Table \ref{tab:sensor_chars}.
For more detailed tuning parameters, please refer to the configuration files provided in the open-source repository.

\begin{table*}
    \centering
    \caption{The configuration of sensor parameters used across the experiments.}
\begin{tabular}{|lr|c|c|c|}
    \hline
    Sensor Characteristic & [Unit]
    & {EuRoC}                            & {UZH FPV}           & Simulation                         \\
    \hline
    gyroscope random walk        & [ rad/s/$\sqrt{\mathrm{Hz}}$     ] & $2.43 \times 10^{-4}$ & $1.19 \times 10^{-3}$ & $1.0 \times 10^{-4}$ \\
    accelerometer random walk    & [ m/s/$\sqrt{\mathrm{Hz}}$       ] & $1.24 \times 10^{-2}$ & $3.26 \times 10^{-5}$ & $1.0 \times 10^{-3}$ \\
    gyroscope bias diffusion     & [ rad/s$^2$/$\sqrt{\mathrm{Hz}}$ ] & $1.34 \times 10^{-4}$ & $2.00 \times 10^{-4}$ & $1.0 \times 10^{-5}$ \\
    accelerometer bias diffusion & [ m/s$^3$/$\sqrt{\mathrm{Hz}}$   ] & $4.46 \times 10^{-3}$ & $6.30 \times 10^{-3}$ & $1.0 \times 10^{-4}$ \\
    IMU rate                     & [ Hz                             ] & 200                   & 500                   & 500                  \\
    image resolution             & [ px                             ] & $752\times 480$       & $640\times 480$       & $752\times 480$      \\
    camera focal length          & [ px                             ] & (458.654, 457.296)    & (275.460, 274.995)    & (458.654, 457.296)   \\
    image centre                 & [ px                             ] & (367.215, 248.375)    & (315.958, 242.712)    & (367.215, 248.375)   \\
    feature noise std. dev.      & [ px                             ] & 1.9                   & 3.8                   & 0.5                  \\
    camera frame rate            & [ Hz                             ] & 20                    & 30                    & 30                   \\
    \hline
\end{tabular}
\label{tab:sensor_chars}
\end{table*}

Table \ref{tab:rmse-euroc} lists the root mean square error (RMSE) of the position estimates of each of the systems in Table \ref{tab:algorithms} on the EuRoC dataset.
EqVIO achieves the best performance in five of the sequences, and also achieves the best mean performance alongside OpenVINS.
Additionally, Table \ref{tab:timing-euroc} lists the processing time taken per frame for each of the algorithms tested on the authors' desktop.
Table \ref{tab:config-euroc} lists the state length parameters used to achieve the results for each of the algorithms.
On average, EqVIO is faster than the next fastest algorithm by a factor of 2.14 on the EuRoC dataset.

\begin{table}[htb]
\begin{minipage}{0.50\textwidth}    
\centering
\caption{RMSE of position estimates in metres on the EuRoC \cite{2016_burri_EuRoCMicroAerial} dataset for each of the systems in Table \ref{tab:algorithms}.}
\label{tab:rmse-euroc}
\begin{tabular}{c|cccccc}
    Algorithm & \rotatebox{\algorithimRotate}{EqVIO} & \rotatebox{\algorithimRotate}{ROVIO} & \rotatebox{\algorithimRotate}{OpenVINS} & \rotatebox{\algorithimRotate}{VINS-Mono} & \rotatebox{\algorithimRotate}{MSCKF} & \rotatebox{\algorithimRotate}{OKVIS} \\
    \hline
    Source & * & \cite{2018_delmerico_BenchmarkComparisonMonocular} & * & * & \cite{2018_delmerico_BenchmarkComparisonMonocular} & \cite{2018_delmerico_BenchmarkComparisonMonocular} \\
    \hline
    MH\_01 & 0.14            & 0.21            & \hlresult{0.09} & 0.24            & 0.42 & 0.16 \\
    MH\_02 & \hlresult{0.20} & 0.25            & 0.24            & 0.21            & 0.45 & 0.22 \\
    MH\_03 & \hlresult{0.09} & 0.25            & 0.16            & 0.15            & 0.23 & 0.24 \\
    MH\_04 & \hlresult{0.22} & 0.49            & 0.31            & 0.23            & 0.37 & 0.34 \\
    MH\_05 & \hlresult{0.28} & 0.52            & 0.39            & \hlresult{0.28} & 0.48 & 0.47 \\
    V1\_01 & \hlresult{0.06} & 0.10            & 0.08            & 0.07            & 0.34 & 0.09 \\
    V1\_02 & 0.14            & 0.10            & \hlresult{0.09} & 0.28            & 0.20 & 0.20 \\
    V1\_03 & 0.19            & 0.14            & \hlresult{0.06} & 0.16            & 0.67 & 0.24 \\
    V2\_01 & 0.11            & 0.20            & 0.08            & \hlresult{0.07} & 0.10 & 0.13 \\
    V2\_02 & 0.17            & 0.14            & \hlresult{0.07} & 0.16            & 0.16 & 0.16 \\
    V2\_03 & 0.20            & \hlresult{0.14} & 0.15            & 0.22            & 1.13 & 0.29 \\
    \hline
    Mean   & \hlresult{0.16} & 0.23            & \hlresult{0.16} & 0.19            & 0.41 & 0.23 \\
\end{tabular}
\end{minipage}
\hspace*{0.05\textwidth}
\begin{minipage}{0.35\textwidth}
\caption{Average time taken to process each frame in ms on the EuRoC \cite{2016_burri_EuRoCMicroAerial} dataset for some of the systems in Table \ref{tab:algorithms}.}
\label{tab:timing-euroc}
\centering
\begin{tabular}{c|cccccc}
    Algorithm & \rotatebox{\algorithimRotate}{EqVIO} & \rotatebox{\algorithimRotate}{ROVIO} & \rotatebox{\algorithimRotate}{OpenVINS} & \rotatebox{\algorithimRotate}{VINS-Mono} \\
    \hline
    Source & * & * & * & * \\
    \hline
    MH\_01 & \hlresult{5.13} & 13.42 & 10.27 & 36.61 \\
    MH\_02 & \hlresult{4.87} & 15.51 & 10.84 & 34.86 \\
    MH\_03 & \hlresult{5.99} & 16.10 & 11.92 & 34.35 \\
    MH\_04 & \hlresult{5.88} & 19.27 & 11.40 & 34.18 \\
    MH\_05 & \hlresult{5.36} & 19.38 & 11.46 & 35.87 \\
    V1\_01 & \hlresult{5.87} & 19.91 & 13.45 & 33.92 \\
    V1\_02 & \hlresult{4.78} & 21.74 & 12.70 & 26.31 \\
    V1\_03 & \hlresult{5.20} & 18.02 & 11.07 & 24.84 \\
    V2\_01 & \hlresult{5.74} & 35.99 & 12.01 & 42.78 \\
    V2\_02 & \hlresult{5.33} & 19.94 & 11.59 & 28.10 \\
    V2\_03 & \hlresult{5.59} & 20.70 & 10.52 & 22.99 \\
    \hline
    Mean & \textbf{5.43} & 20.00 & 11.57 & 32.26
\end{tabular}
\end{minipage}
\end{table}

\begin{table}
    \caption{State length parameters of the VIO algorithms (EuRoC)}
    \label{tab:config-euroc}
    \centering
\begin{tabular}{c|c|c|c}
    Algorithm & Ref.                                                & \# features         & \# states \\
    \hline
    EqVIO     & *                                                   & 40                  & 1         \\
    ROVIO     & \cite{2017_bloesch_IteratedExtendedKalman}          & 25                  & 1         \\
    OpenVINS  & \cite{2020_geneva_OpenVINSResearchPlatform}         & 50(SLAM) + 200(MSC) & 11        \\
    VINS-mono & \cite{2018_qin_VinsmonoRobustVersatile}             & 150                 & variable  \\
    MSCKF     & \cite{2007_mourikis_MultistateConstraintKalman}     & variable            & 30        \\
    OKVIS     & \cite{2015_leutenegger_KeyframebasedVisualInertial} & 400                 & 5         \\
\end{tabular}
\end{table}

Table \ref{tab:rmse-uzhfpv} lists the RMSE of the position estimates for a subset of the systems in Table \ref{tab:algorithms} on some sequences from the UZH FPV dataset.
To ensure a fair comparison, the sequences considered here are those for which OpenVINS has publicly listed tuning parameters at the time of writing.
The default tuning parameters were used for ROVIO, and its performance could perhaps be improved with tuning specific to the challenging dataset.
The tuning parameters for VINS-Mono were taken from those used in VINS-Stereo (\url{https://github.com/rising-turtle/VINS-Stereo}) but with the second camera disabled.
Over the selected sequences, EqVIO achieves the best performance in four of the ten sequences, and achieves the best mean performance overall.
Additionally, from Table \ref{tab:timing-uzhfpv} it is clear that EqVIO is also significantly faster than any of the other algorithms considered.
Table \ref{tab:config-uzhfpv} lists the state length parameters used to achieve the results for each of the algorithms.
On average, EqVIO is faster than the next fastest algorithm by a factor of 5.3 on the UZH FPV dataset.

\begin{table}[htb]
\begin{minipage}{0.45\textwidth}    
\centering
\caption{RMSE of position estimates in metres on the UZH FPV \cite{2019_delmerico_AreWeReady} dataset for some of the systems in Table \ref{tab:algorithms}.}
\label{tab:rmse-uzhfpv}
\begin{tabular}{c|ccccc}
    Algorithm & \rotatebox{\algorithimRotate}{EqVIO} & \rotatebox{\algorithimRotate}{ROVIO} & \rotatebox{\algorithimRotate}{OpenVINS} & \rotatebox{\algorithimRotate}{VINS-Mono}
    \\ \hline
    Source              & *               & *    & *               & *               \\
    \hline
    indoor\_45\_12      & 0.26            & x    & \hlresult{0.24} & 0.39            \\
    indoor\_45\_13      & \hlresult{0.28} & 1.53 & 0.29            & 0.74            \\
    indoor\_45\_14      & 0.27            & 2.63 & \hlresult{0.22} & x               \\
    indoor\_45\_2       & 0.31            & x    & \hlresult{0.15} & 0.62            \\
    indoor\_45\_4       & 0.33            & x    & \hlresult{0.28} & 1.11            \\
    indoor\_forward\_10 & \hlresult{0.15} & 1.23 & 0.34            & 0.53            \\
    indoor\_forward\_5  & 0.29            & 2.43 & 0.16            & \hlresult{0.12} \\
    indoor\_forward\_6  & 0.23            & 0.65 & \hlresult{0.17} & 0.28            \\
    indoor\_forward\_7  & \hlresult{0.40} & 1.09 & 0.42            & 0.54            \\
    indoor\_forward\_9  & \hlresult{0.17} & x    & 0.50            & 0.56            \\
    \hline
    Mean                & \hlresult{0.27} & 1.59 & 0.28            & 0.54
    \\
\end{tabular}
\end{minipage}
\hspace*{0.05\textwidth}
\begin{minipage}{0.35\textwidth}
    \caption{Average time taken to process each frame in ms on the UZH FPV \cite{2019_delmerico_AreWeReady} dataset for some of the systems in Table \ref{tab:algorithms}.}
    \label{tab:timing-uzhfpv}
    \centering
    \begin{tabular}{c|cccccc}
        Algorithm & \rotatebox{\algorithimRotate}{EqVIO} & \rotatebox{\algorithimRotate}{ROVIO} & \rotatebox{\algorithimRotate}{OpenVINS} & \rotatebox{\algorithimRotate}{VINS-Mono} \\
        \hline
        Source & * & * & * & * \\
        \hline
        indoor\_45\_12         & \hlresult{3.06} & 21.20 & 17.47 & 79.37 \\
        indoor\_45\_13         & \hlresult{3.24} & 28.18 & 18.29 & 54.61 \\
        indoor\_45\_14         & \hlresult{3.53} & 23.85 & 18.53 & 43.16 \\
        indoor\_45\_2          & \hlresult{3.66} & 32.21 & 18.06 & 53.02 \\
        indoor\_45\_4          & \hlresult{3.39} & 24.48 & 18.06 & 58.34 \\
        indoor\_forward\_10    & \hlresult{3.30} & 41.86 & 16.27 & 66.43 \\
        indoor\_forward\_5     & \hlresult{3.33} & 32.77 & 17.35 & 91.71 \\
        indoor\_forward\_6     & \hlresult{3.09} & 43.71 & 17.00 & 45.64 \\
        indoor\_forward\_7     & \hlresult{3.21} & 30.36 & 17.03 & 59.24 \\
        indoor\_forward\_9     & \hlresult{3.02} & 24.34 & 16.20 & 74.21 \\
        \hline
        Mean &                   \textbf{3.28} & 30.30 & 17.43 & 64.02
    \end{tabular}
\end{minipage}
\end{table}

\begin{table}
    \caption{State length parameters of the VIO algorithms (UZH FPV)}
    \label{tab:config-uzhfpv}
    \centering
    \begin{tabular}{c|c|c|c}
        Algorithm & Ref. & \# features & \# states \\
        \hline
        EqVIO & * & 40 & 1 \\
        ROVIO & \cite{2017_bloesch_IteratedExtendedKalman} & 25 & 1 \\
        OpenVINS & \cite{2020_geneva_OpenVINSResearchPlatform} & 50(SLAM) + 200(MSC) & 11 \\
        VINS-mono & \cite{2018_qin_VinsmonoRobustVersatile} & 300 & variable \\
    \end{tabular}
\end{table}

One reason for the faster processing time of EqVIO is the use of fewer features than other algorithms.
Only ROVIO uses fewer features than EqVIO in these experiments, but this is because it uses a direct image difference as its measurement error, which yields more information from each feature at the cost of increased processing time.
While the other algorithms could be made faster by reducing the number of features used, this generally reduces the accuracy of the estimated trajectory, as is also seen to be the case for EqVIO in Section \ref{sec:landmark_experiment}.
While it remains an open question how to compare different VIO algorithms when accounting for different tuning paramaters, the results in this Section have shown that EqVIO is able to achieve competitive accuracy while processing data at much higher rates than other state-of-the-art algorithms.

\subsection{Verification of Consistency}
\label{sec:consistency}

In order to verify the consistency of the proposed system, we performed a series of Monte Carlo simulations and computed statistics of the Normalised Estimation Error Squared (NEES) and the filter covariance.
The NEES of the EqF is calculated using the formula
\begin{align*}
    \mathrm{NEES} = \frac{1}{m} \vartheta(\phi(\hat{X}^{-1}, \xi))^\top \Sigma^{-1} \vartheta(\phi(\hat{X}^{-1}, \xi)),
\end{align*}
where $m = 21+3n$ is the dimension of $\vinsT$, $\hat{X} \in \grpG$ is the observer state, $\xi \in \calM$ is the true system state, and $\Sigma \in \Sym_+(21+3n)$ is the EqF Riccati matrix.

\begin{figure}[htb]
    \centering
    \includegraphics[width=0.7\linewidth]{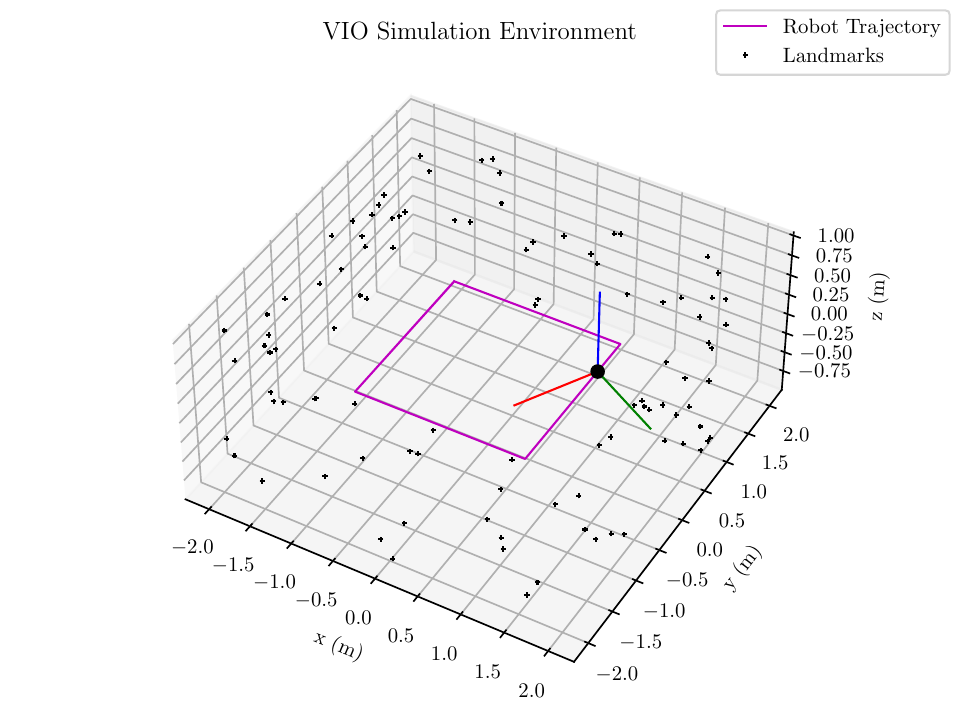}
    \caption{An example of the simulation environment used for experiments to verify consistency.}
    \label{fig:square_pattern}
\end{figure}

In each simulation, the robot trajectory was defined by a square pattern of side length 2 that is flown every 20 seconds.
The landmarks were scattered uniformly on the four vertical walls around the square pattern at a distance of 1~m, with 25 landmarks scattered on each wall.
Figure \ref{fig:square_pattern} shows an example of the true trajectory and the true landmark positions.
Every simulation was run for a total of 200~s, meaning the robot flew around the edges of the square exactly 10 times.
The characteristics of the simulated IMU and camera were chosen as in Table \ref{tab:sensor_chars}.
EqVIO was configured with the corresponding input and output gain matrices, and the state gain matrix was set to zero.
Additionally, outlier rejection was not used, and new landmarks were simply added and removed according to their visibility in the simulation.

Figure \ref{fig:square_nees} shows the resulting statistics of the NEES of EqVIO taken over 1000 Monte Carlo simulations.
NEES statistics are shown for estimates of the full state, the pose, and the attitude.
The top row of the figure shows the experimental and theoretical median values and 95\% confidence bounds of NEES for each time.
In all three subplots the results of the experimental trials closely match the theoretical values.
The bottom row of the figure shows the experimental and theoretical distributions of NEES at the final simulation time of 200~s.
Again, in all three subplots the histogram of experimental values matches the theoretical distribution closely.

\begin{figure*}[htb]
\centering
\includegraphics[width=1.0\linewidth]{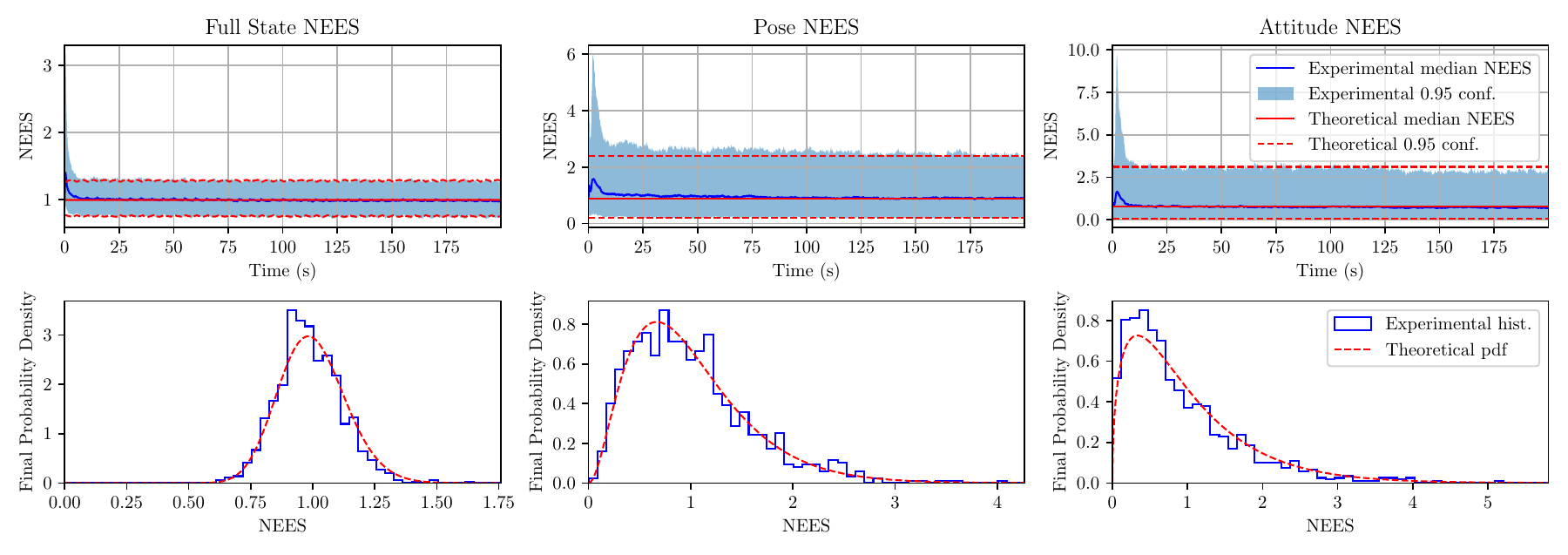}
\caption{The normalised estimation error squared over 1000 Monte Carlo trials.
The top row shows the median and 95\% confidence bounds for both the theoretical and experimental distributions of NEES over time.
The bottom row shows the experimental (represented as a histogram) and theoretical distribution (represented as a smooth curve) of NEES at 200~s.
From left to right, the columns correspond to the the NEES of the full state, the NEES of just the pose estimate, and the NEES of just the attitude estimate.}
\label{fig:square_nees}
\end{figure*}

Figure \ref{fig:square_bounds} shows the evolution of pose error over time for each trial, along with lines indicating three times the standard deviation reported by the filter.
The errors in roll and pitch converge quickly and remain small for the whole simulation time.
The errors in position and yaw grow slowly unbounded over time due to their unobservability associated with the invariance group action as discussed in Section \ref{sec:invariance}.
Figure \ref{fig:square_bounds} demonstrates that the pose error growth is in line with the estimated associated standard deviation.

\begin{figure}
\centering
\includegraphics[width=0.5\linewidth]{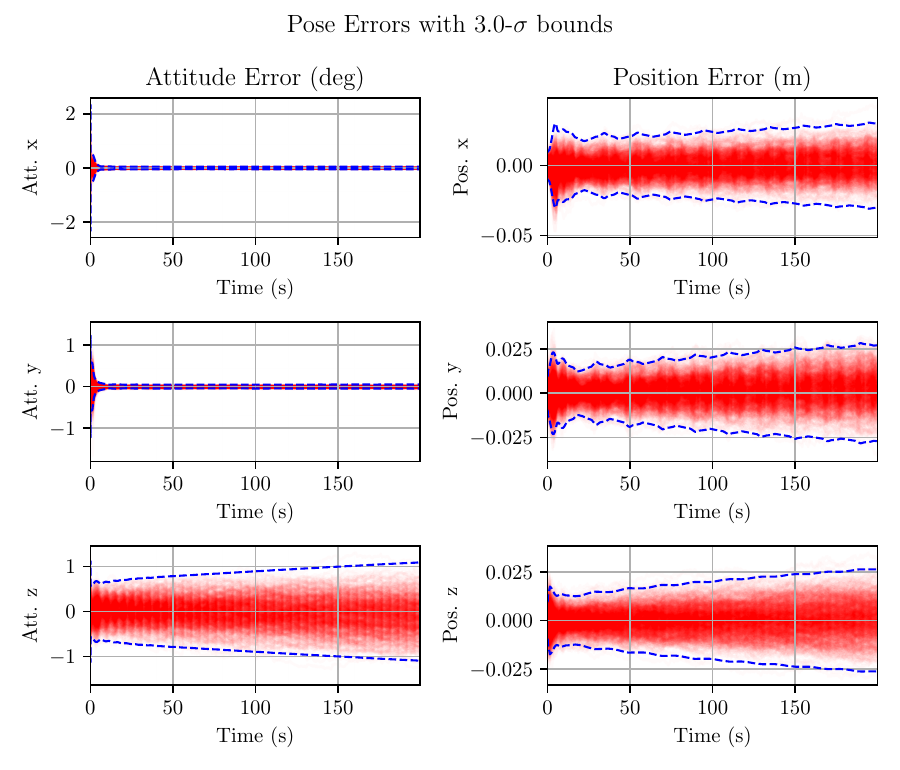}
\caption{The pose error over time for each of the 1000 Monte Carlo trials.
Lines that indicate three times the standard deviation of the filter are also shown and show that the errors do not grow faster than expected.}
\label{fig:square_bounds}
\end{figure}

These results verify that the proposed EqF provides consistent estimates of the true system state.
This is a characteristic shared with other recent works that exploit symmetry for observer design in SLAM and VIO, such as \cite{2016_barrau_EKFSLAMAlgorithmConsistency,2017_wu_InvariantEKFVINSAlgorithm,2018_brossard_InvariantKalmanFiltering,2018_heo_ConsistentEKFBasedVisualInertial}, which also achieve filter consistency by exploiting additional techniques like the first-estimates-Jacobians \cite{2008_huang_AnalysisImprovementConsistency,2020_geneva_OpenVINSResearchPlatform} or the observability-constrained update \cite{2010_huang_ObservabilitybasedRulesDesigning}.
In summary, EqVIO maintains a consistent estimate and covariance of the true state error over time, and the observability of the system states is reflected accurately in the Riccati matrix.

\subsection{Convergence of Bias and Camera Parameters}
\label{sec:cambias}

We conducted an additional series of Monte Carlo experiments to verify the convergence of the camera extrinsics and IMU biases to their true values.
In order to ensure observability of the camera extrinsics, in particular, a different trajectory was chosen to the square trajectory used in the consistency experiments in Section \ref{sec:consistency}.
In each simulation, the IMU trajectory was defined by
\begin{align*}
    R_\imu(t) &:= \exp_{\SO(3)} \left(\frac{\pi}{4}
    \begin{pmatrix}
        \cos(0.25 * t) \\
        \cos(-0.3 * t) \\
        \cos(0.2 * t) \\
    \end{pmatrix}^\times
    \right), \\
    x_\imu(t) &:= \frac{1}{2}
    \begin{pmatrix}
         \cos(0.1 \pi t) \\
         \cos(0.2 \pi t) \\
         \cos(0.15 \pi t) \\
    \end{pmatrix}.
\end{align*}
The landmarks were scattered uniformly on the four vertical walls and on the floor and ceiling around the bounds of the trajectory at a distance of 1~m from the bounds of the trajectory, with 20 landmarks scattered on each surface.
Figure \ref{fig:sine_pattern} shows an example of the true trajectory and the true landmark positions.
Every simulation was run for a total of 90~s.
Other than this the simulation configuration was the same as in Section \ref{sec:consistency}.
The characteristics of the simulated IMU and camera were defined as in Table \ref{tab:sensor_chars}.
EqVIO was configured with the corresponding input and output gain matrices, the state gain matrix was set to zero, outlier rejection was not used, and new landmarks included in the state depending on their visibility in the simulation.

\begin{figure}[htb]
    \centering
    \includegraphics[width=0.7\linewidth]{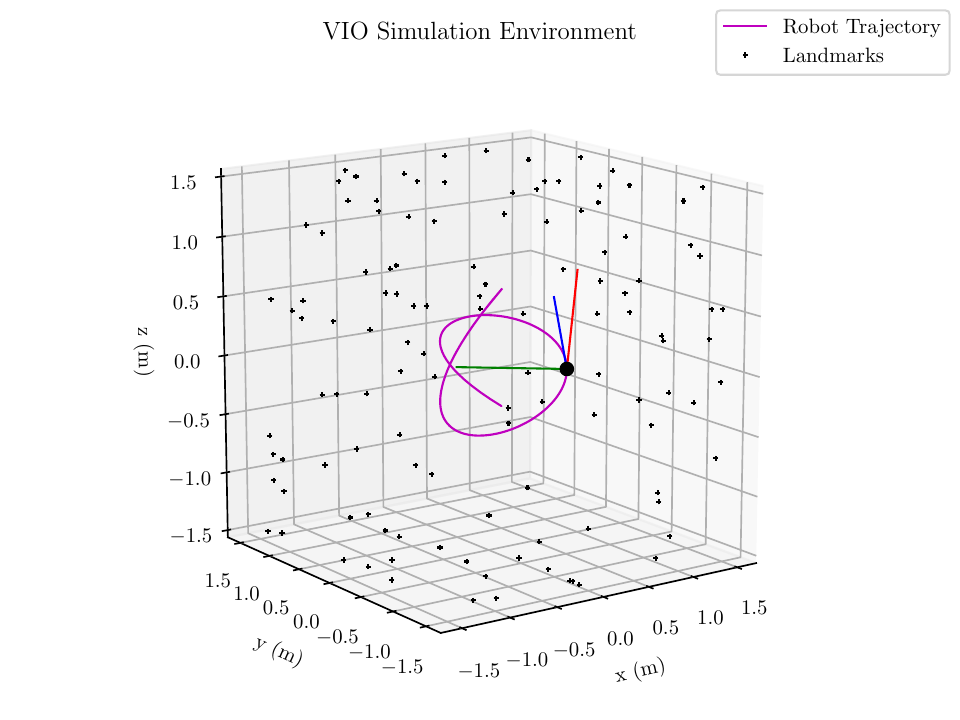}
    \caption{An example of the simulation environment used for experiments to verify camera and bias convergence.}
    \label{fig:sine_pattern}
\end{figure}

Figure \ref{fig:camera_convergence} shows the convergence of the camera extrinsics errors along with bounds indicating three times the standard deviation reported by the filter.
The initial extrinsics rotation errors and translation errors were drawn from zero-mean normal distributions.
The initial standard deviations of rotation error and translation error were chosen to be $\sqrt{5}\times 10^{-2}$~rad and $0.05$~m, respectively.
The figure shows that the camera extrinsics successfully converge from each of these initial errrors, and that the errors are consistent with the covariance reported by the filter.
This verifies the ability of EqVIO to perform online extrinsics calibration.

Figure \ref{fig:bias_convergence} shows the evolution of bias error over time for each trial, along with lines indicating three times the standard deviation reported by the filter.
The initial values of the gyroscope bias and accelerometer bias were drawn from zero-mean normal distributions with standard deviations 0.3~rad/s and 0.1~m/s$^2$, respectively.
The figure shows that the gyroscope and accelerometer biases all converge successfully and remain within the bounds suggested by the filter covariance.

\begin{figure}
    \centering
    \begin{minipage}{0.45\textwidth}
        \centering
        \includegraphics[width=1.0\linewidth]{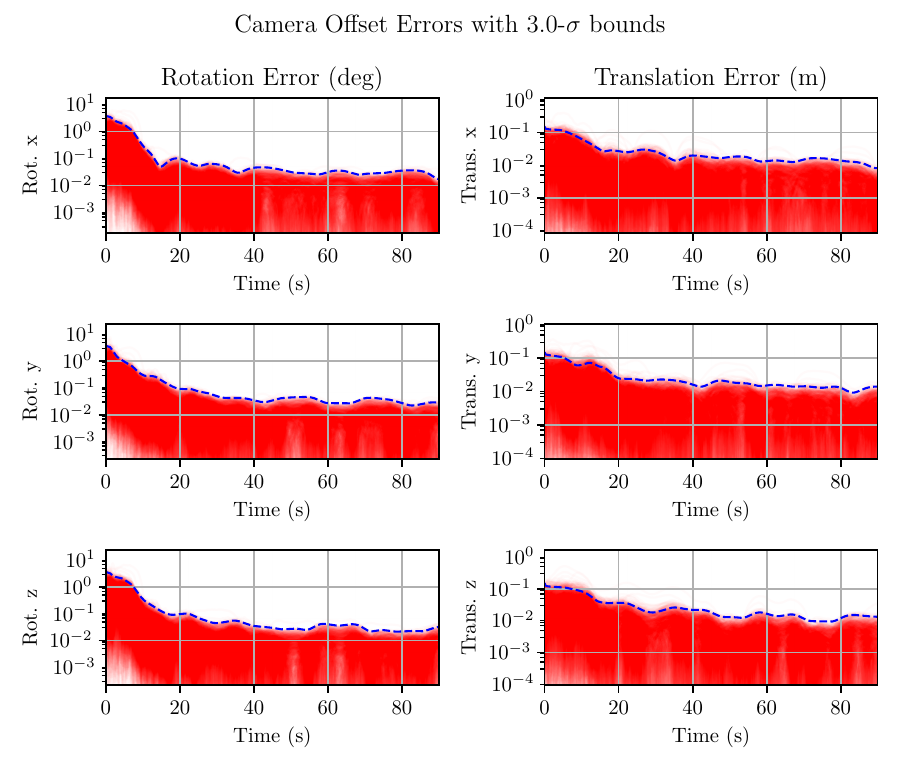}
        \caption{The absolute camera extrinsics errors over time for each of the 1000 Monte Carlo trials.
        Lines that indicate three times the standard deviation reported by the filter are also shown, and the true errors appear consistent with these bounds.}
        \label{fig:camera_convergence}
    \end{minipage}
    \hspace*{0.05\textwidth}
    \begin{minipage}{0.45\textwidth}
        \centering
        \includegraphics[width=1.0\linewidth]{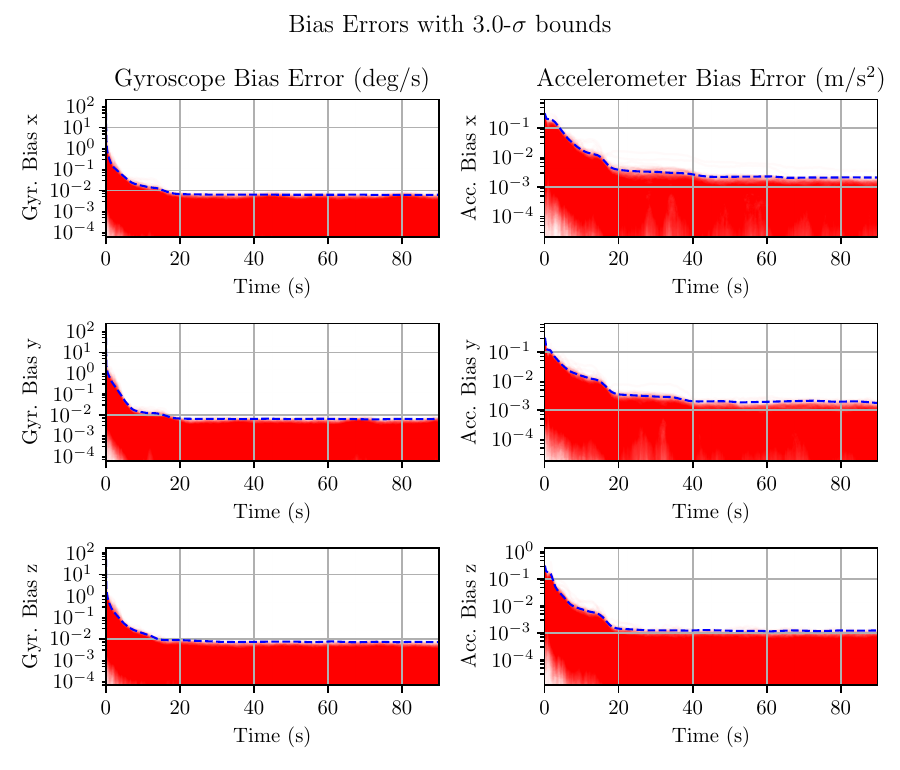}
        \caption{The absolute errors in estimation of IMU biases over time for 1000 Monte Carlo trials.
            The dashed blue lines represent three times the standard deviation as reported by the EqF, and it can be seen that the experimental bias errors are consistent with these.
        }
        \label{fig:bias_convergence}
    \end{minipage}
\end{figure}

The results in this section demonstrate that EqVIO is able to estimate camera extrinsics and IMU biases even with reasonably poor initial estimates.
The camera translation and accelerometer bias are seen to take longer to converge than the camera rotation and gyroscope bias.
This is explained by the inherent observability of VIO: accurate estimation of the camera translation offset depends on the IMU undergoing rotations about at least two axes.
The rates of convergence are also all reflected in the estimated covariance; the errors are consistent with the standard deviations reported by the EqF.
Overall, these results show that EqVIO can be used to accurately calibrate camera offsets and IMU biases online, and that the reported uncertainty of these estimates matches their true distributions.

\subsection{Example Performance Details}

\newcommand{\exampleSequence}{EuRoC sequence V2\_01}

In addition to the experiments comparing EqVIO's performance with other state-of-the-art algorithms, we collected data to evaluate and verify the system's performance.
We provide examples of these additional results on the \exampleSequence.

Figure \ref{fig:example_seq_flamegraph} shows the time taken to process every frame of the sequence as a flamegraph, and Figure \ref{fig:example_seq_time_histogram} shows histograms of the time taken for each key step of the system.
These figures show that the time taken by EqVIO is significantly increased when new features need to be identified rather than only tracked.
Nonetheless, the peak time taken for any frame is still comparable to the average processing time of the other algorithms listed in Table \ref{tab:timing-euroc}.

\begin{figure}[!htb]
    \centering
    \begin{minipage}{0.45\textwidth}    
        \includegraphics[width=1.0\linewidth]{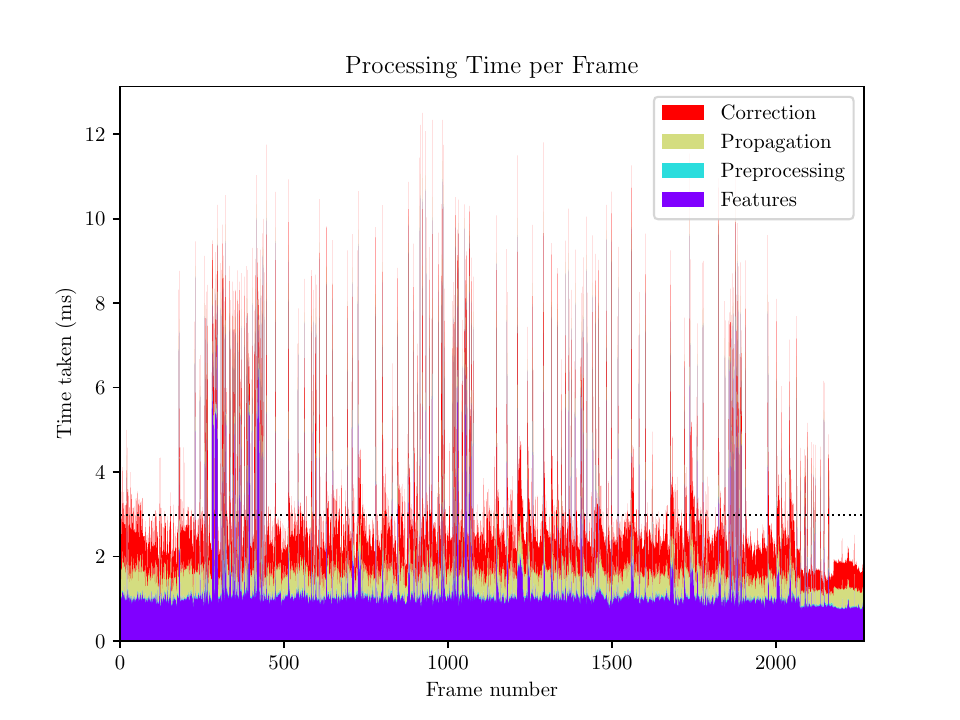}
        \caption{A flame graph showing the time taken to process each frame of the \exampleSequence.}
        \label{fig:example_seq_flamegraph}
    \end{minipage}
    \hspace*{0.05\textwidth}
    \begin{minipage}{0.45\textwidth}
        \includegraphics[width=1.0\linewidth]{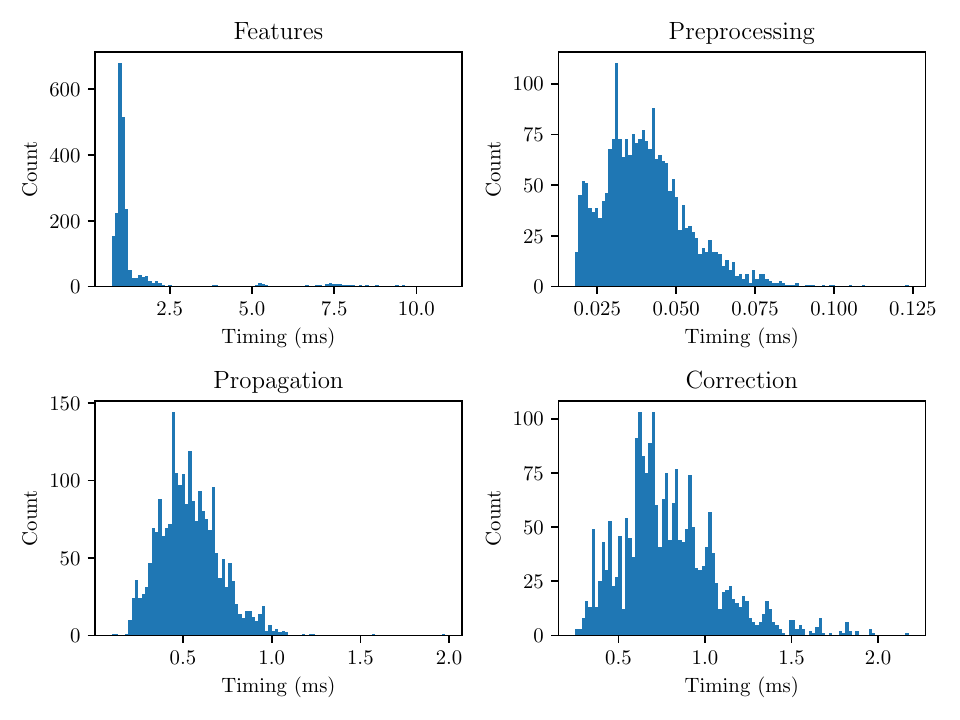}
        \caption{Histograms of the time taken to process each frame by each section of EqVIO in the \exampleSequence.}
        \label{fig:example_seq_time_histogram}
    \end{minipage}
\end{figure}

The absolute position and the attitude about the direction of gravity are unobservable due to the reference frame invariance described in Proposition \ref{prop:invariance_action}.
However, the velocity and direction of gravity with respect to the body-fixed frame are observable, and thus expected to remain free of drift for all time.
Figure \ref{fig:example_seq_gravvel} shows the estimated and true values of the body-fixed gravity and linear velocity $R_\imu^\top v_\imu$ over time.
Clearly the EqF maintains a highly accurate estimate of both the body-fixed gravity direction and velocity over the whole trajectory, and, as expected, no drift is present.

\begin{figure}[!htb]
    \centering
    \includegraphics[width=0.7\linewidth]{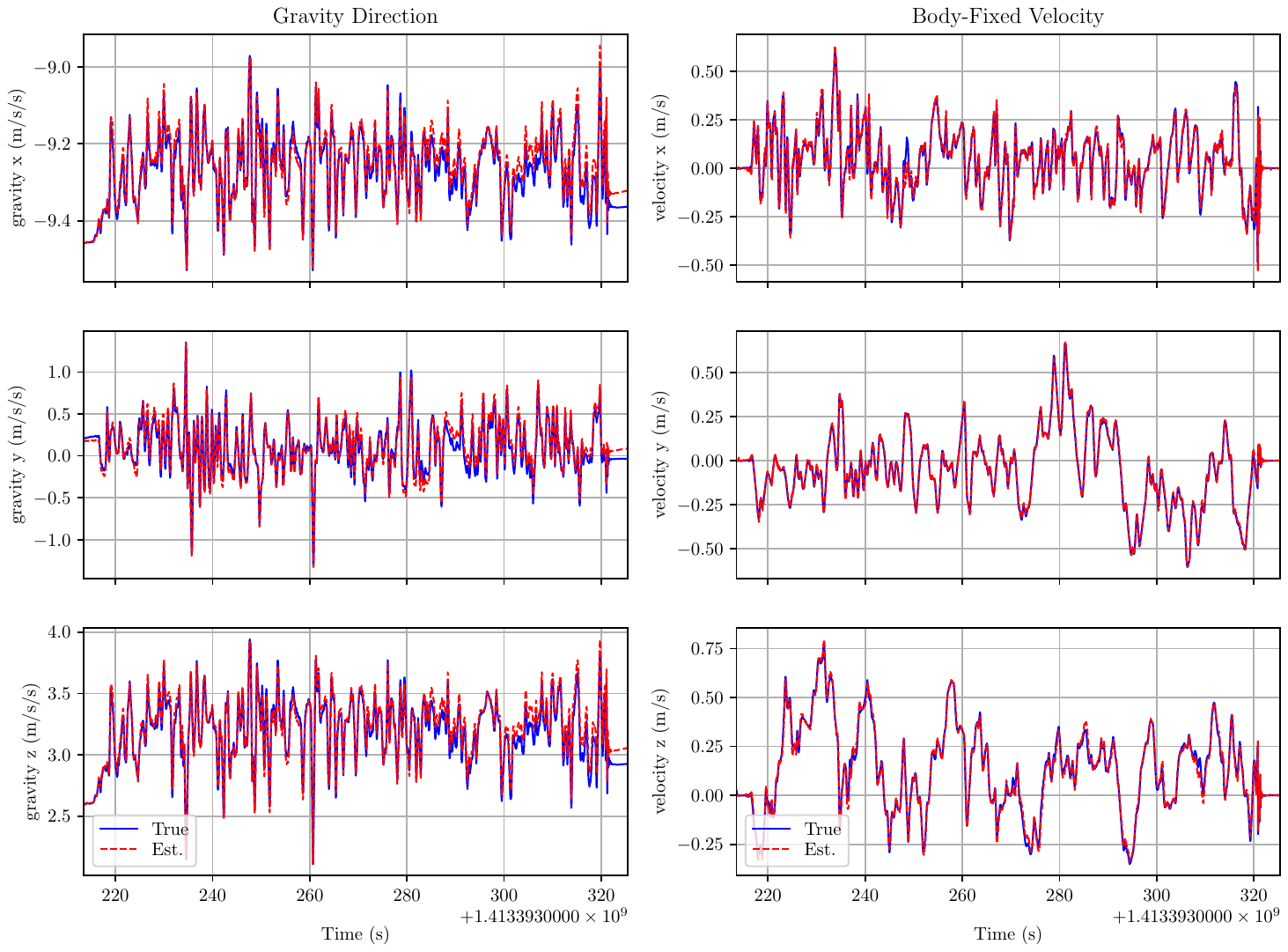}
    \caption{Comparison of estimated and true values of body-fixed velocity and gravity in the \exampleSequence.}
    \label{fig:example_seq_gravvel}
\end{figure}

Figure \ref{fig:example_seq_biases} shows the estimated IMU biases over time.
The estimated gyroscope biases converge quickly and are very stable throughout the sequence, while the estimated accelerometer biases vary significantly over time.
This is associated with observability properties of the VIO problem, and is also reflected in the bias convergence experiments of Section \ref{sec:cambias}.
According to the estimated groundtruth provided with the \exampleSequence, the true gyroscope and accelerometer biases are constant values $b_\imu^\Omega = (-0.002295,0.024939,0.081667)$~rad/s and $b_\imu^a = (-0.023601,0.121044,0.074783)$~m/s$^2$, respectively.
The estimates shown in Figure \ref{fig:example_seq_biases} are all close to these values once converged, with the exception of the $z$-axis of the accelerometer bias.
One factor may be that EqVIO approximates the strength of gravity as $9.80665$~m/s$^2$, rather than using the exact gravity of the room where the dataset was recorded.

\begin{figure}[!htb]
    \centering
    \includegraphics[width=0.7\linewidth]{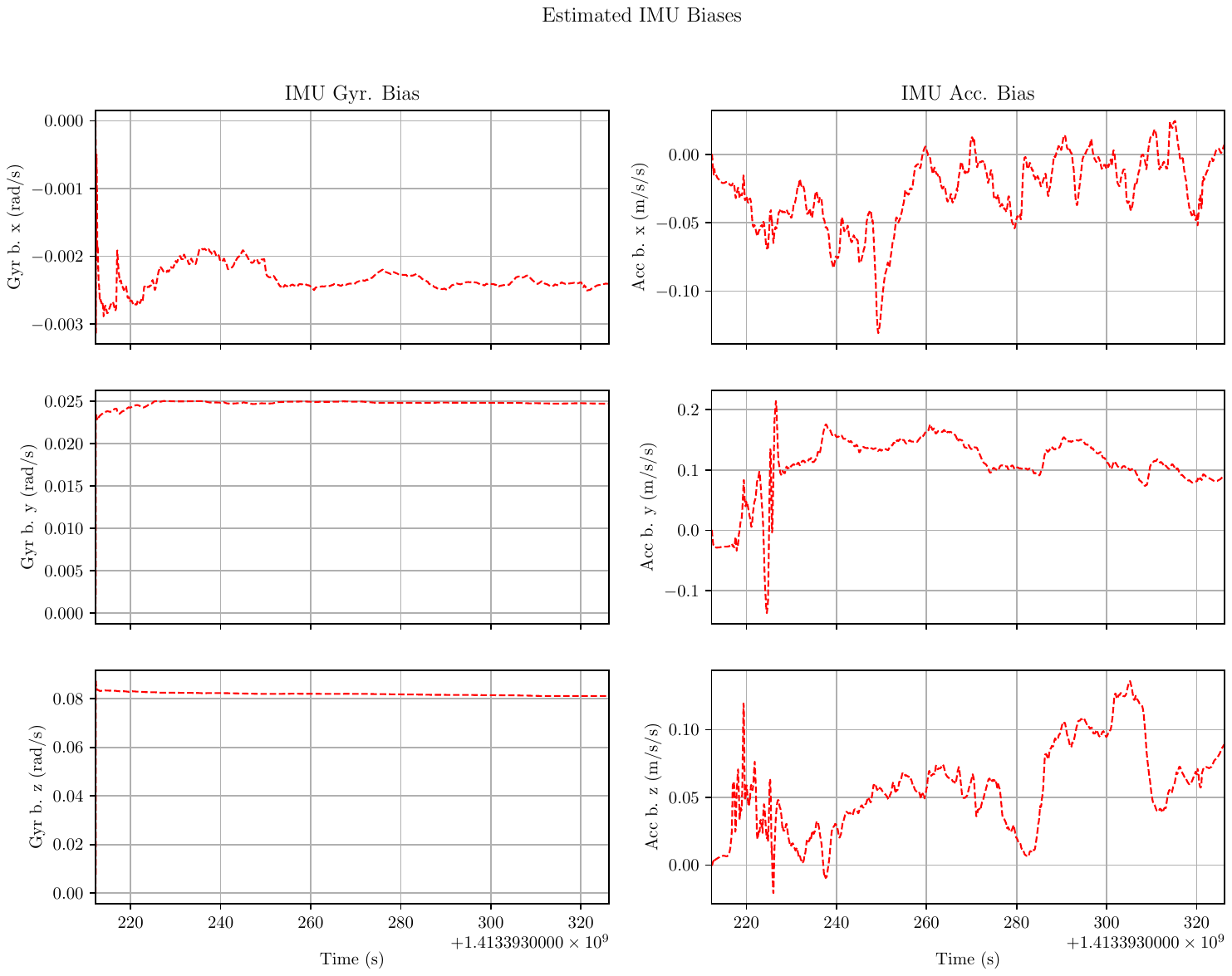}
    \caption{The estimated IMU biases over time in the \exampleSequence.}
    \label{fig:example_seq_biases}
\end{figure}

\subsection{Number of Landmarks}
\label{sec:landmark_experiment}

In the experiments carried out on the EuRoC and UZH FPV datasets, EqVIO was restricted to use a maximum of 40 landmarks at any given time.
This is a design choice based on trading off the desired accuracy and processing time performance.
Figure \ref{fig:landmark_experiment} shows how the position RMSE and processing time of EqVIO change for varying numbers of landmarks.
EqVIO was run over the entire EuRoC dataset with a single set of parameters, except that the maximum number of landmarks allowed was changed from 4 to 100 in increments of 2 for each trial.
The processing time per frame and position RMSE were averaged over the whole dataset for each trial.
Due to the similarities between the EqF and an EKF, the processing time indeed appears to increase quadratically with the number of landmarks considered.
In contrast, the position RMSE decreases as the number of landmarks increases, although this trend is less consistent.
The experiment shows that the number of landmarks used in EqVIO may be increased to yield improved accuracy at the cost of increased processing time.
This trade-off is useful in allowing a practitioner to tune the proposed system according to their requirements.

\begin{figure}
    \centering
    \includegraphics[width=0.7\linewidth]{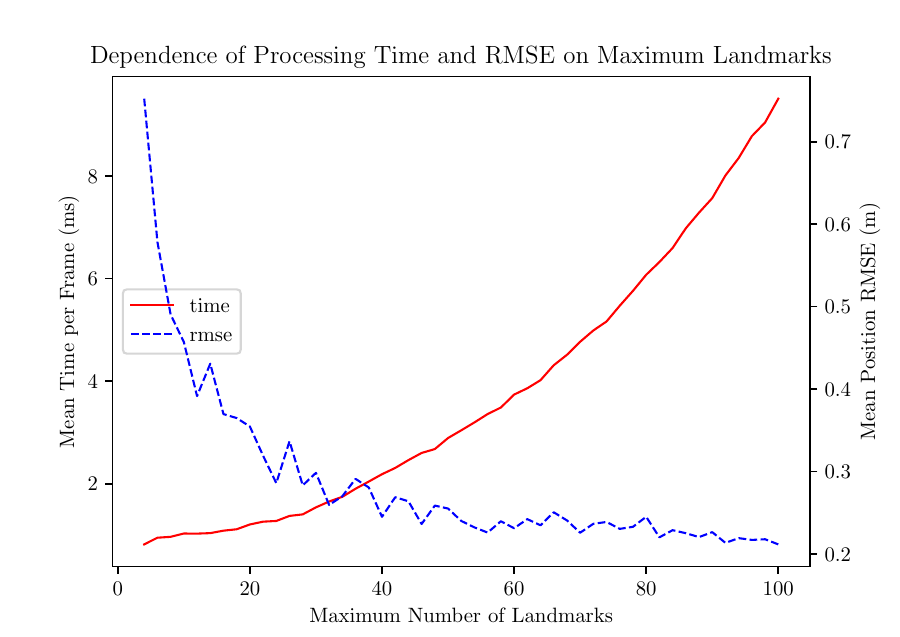}
    \caption{The relationship between the maximum number of landmarks used in EqVIO and the position RMSE and processing time on the EuRoC dataset.}
    \label{fig:landmark_experiment}
\end{figure}

\section{Conclusion}

This paper presents and develops EqVIO: a novel system for visual-inertial odometry based on the recently proposed equivariant filter.
A new Lie group, the VI-SLAM group, is developed for the VIO problem.
It is shown that this symmetry is compatible with the well-known reference frame invariance of VIO, and that therefore the resulting EqF is a naturally consistent estimator.
The VI-SLAM group incorporates the extended special Euclidean group $\SE_2(3)$ proposed in \cite{2014_barrau_InvariantParticleFiltering}, leading to exact linearisation of the error dynamics associated with the navigation states.
The VI-SLAM group also takes advantages of the $\SOT(3)$ symmetry of visual landmarks to enable the EqF to use the higher-order equivariant output approximation \cite{2023_vangoor_EquivariantFilterEqF}.
EqVIO is the system that combines an EqF based on the VI-SLAM group with a feature tracking front-end.
In experimental results on two popular VIO datasets \cite{2016_burri_EuRoCMicroAerial,2019_delmerico_AreWeReady}, we show that EqVIO outperforms other state-of-the-art algorithms in terms of both speed of computation and accuracy of trajectory estimation.
In summary, this paper shows the compatibility of the novel VI-SLAM Lie group with the VIO problem, and that the resulting equivariance-based VIO system significantly outperforms alternative solutions.


\section*{Appendix}

\begin{proof}[Proof of Proposition \ref{prop:invariance_action}]
Let $S \in \SE_{\eb_3}(3)$, $\xi = (\xi_\imu, b_\imu, T, p_i) \in \vinsT$, and $(\Omega, a) \in \vecL$ be arbitrary.
Then compute
\begin{align*}
    &f_{(\Omega, a)} (\alpha(S, (\xi_\imu, b_\imu, T, p_i))) \\
    &= f_{(\Omega, a)} ((S^{-1} P_\imu, R_S^\top v_\imu, b_\imu), T, S^{-1}(p_i)), \\
    &= (
        (R_S^\top R_\imu) (\Omega - b^\Omega_\imu)^\times,
        (R_S^\top v_\imu),
        (R_S^\top R_\imu) (a - b^a_\imu) + g \eb_3,
        \\ & \hspace{1cm}
        0, 0, 0
    ), \\
    &= (
        R_S^\top (R_\imu (\Omega - b^\Omega_\imu)^\times),
        R_S^\top v_\imu,
        R_S^\top (R_\imu (a - b^a_\imu) + g \eb_3),
        \\ & \hspace{1cm}
        0, 0, 0,
    ), \\
    &= \td \alpha_S f_{\Omega, a} (\xi_\imu, b_\imu, T, p_i),
\end{align*}
where the second-last line follows from $R_S^\top \eb_3 = \eb_3$.
This shows that, indeed, $f$ is invariant with respect to the action $\alpha$.
To show the invariance of $h$, it is sufficient to show the invariance of the component functions $h^k$ defined in \eqref{eq:measurement_function}.
One has that
\begin{align*}
    &h^k(\alpha(S, (\xi_\imu, b_\imu, T, p_i))) \\
    &= h^k (((S^{-1} P_\imu, R_S^\top v_\imu, b^\Omega_\imu, b^a_\imu), T, S^{-1}(p_i))), \\
    &= \pi_{\Sph^2} \left( (S^{-1} P_\imu T)^{-1} (S^{-1}(p_i)) \right), \\
    &= \pi_{\Sph^2} \left( (P_\imu T)^{-1} S S^{-1}(p_k) \right), \\
    &= \pi_{\Sph^2} \left( (P_\imu T)^{-1}(p_k) \right), \\
    &= h^k(\xi_\imu, b_\imu, T, p_i)).
\end{align*}
This completes the proof.
\end{proof}

\begin{proof}[Proof of Lemma \ref{lem:simple_landmark_equivariance}]
It is straightforward to see that $\varphi^\vis$ and $\rho^\vis$ are indeed right group actions.
To see the equivariance of $h^\vis$, let $Q \in \SOT(3)$ and $q \in \R^3 \setminus \{0\}$ be arbitrary.
Then,
\begin{align*}
    h^\vis  (\varphi^\vis(Q, q))
    &= h^\vis ( c_Q^{-1} R_Q^\top q ), \\
    &= \frac{c_Q^{-1} R_Q^\top q}{\vert c_Q^{-1} R_Q^\top q \vert}, \\
    &= c_Q^{-1} R_Q^\top \frac{q}{c_Q^{-1} \vert  R_Q^\top q \vert}, \\
    &= R_Q^\top \frac{q}{\vert q \vert}, \\
    &= \rho^\vis ( Q, h^\vis(q)),
\end{align*}
as required.
\end{proof}

\begin{proof}[Proof of Lemma \ref{lem:state_action}]
Let $\xi = (\xi_\imu, b_\imu, T, p_i) \in \vinsT$ and $X_1 = (A_1, \beta_1, B_1, Q_{1,i}), X_2 = (A_2, \beta_2, B_2, Q_{2,i}) \in \grpG$ be arbitrary.
Then,
\begin{align*}
    &\phi(X_2, \phi(X_1, \xi)) \\
    &= \phi(X_2,(\varphi^\imu(A_1, \xi_\imu), b_\imu + \beta_1,
    P_{A_1}^{-1} T B_1,
    \\ &\hspace{1cm}
    P_\imu T B_1 Q_{1,i}^{-1} T^{-1} P_\imu^{-1} (p_i)) ), \\
    &= (\varphi^\imu(A_2, \varphi^\imu(A_1, \xi_\imu)),
    b_\imu + \beta_1 + \beta_2,
    P_{A_2}^{-1} (P_{A_1}^{-1} T B_1) B_2,
    \\ &\hspace{1cm}
    P_\imu P_{A_1} P_{A_1}^{-1} T B_1 B_2 Q_{2,i}^{-1} (P_{A_1}^{-1} T B_1)^{-1} (P_\imu P_{A_1})^{-1}
    \\ &\hspace{1cm}
    (P_\imu T B_1 Q_{1,i}^{-1} T^{-1} P_\imu^{-1} (p_i))), \\
    &= (\varphi^\imu(A_1 A_2, \xi_\imu),
    b_\imu + (\beta_1 + \beta_2),
    (P_{A_1} P_{A_2})^{-1} T (B_1 B_2),
    \\ &\hspace{1cm}
    P_\imu T B_1 B_2 Q_{2,i}^{-1} B_1^{-1} T^{-1} P_{A_1} P_{A_1}^{-1} P_\imu^{-1}
    \\ &\hspace{1cm}
    P_\imu T B_1 Q_{1,i}^{-1} T^{-1} P_\imu^{-1} (p_i)), \\
    &= (\varphi^\imu(A_1 A_2, \xi_\imu),
    b_\imu + (\beta_1 + \beta_2),
    (P_{A_1} P_{A_2})^{-1} T (B_1 B_2),
    \\ &\hspace{1cm}
    P_\imu T B_1 B_2 Q_{2,i}^{-1} Q_{1,i}^{-1} T^{-1} P_\imu^{-1} (p_i)), \\
    &= \phi(X_1 X_2, \xi).
\end{align*}
This shows that the compatibility condition \eqref{eq:group_action_compatible} is satisfied.
For the identify condition \eqref{eq:group_action_identity}, compute
\begin{align*}
    &\phi (\id_{\grpG}, \xi) \\
    &= \phi ((I_5, 0, I_4, (I_4)), (\xi_\imu, b_\imu, T, p_i)), \\
    &= (\varphi^\imu(I_5, \xi_\imu), b_\imu + 0,
        I_4 T I_4, P_\imu T I_4 I_4^{-1} T^{-1} P_\imu^{-1} (p_i)), \\
    &= (\xi_\imu, b_\imu, p_i).
\end{align*}
Then, indeed, $\phi$ is a group action.
Finally, to see that $\phi$ is transitive, let $\xi^1, \xi^2 \in \vinsT$ be arbitrary, and let $X = (A, \beta, B, Q_i) \in \grpG$ such that
\begin{align*}
    P_A &= (P_\imu^1)^{-1} P_\imu^2, \\
    v_A &= (R_\imu^1)^\top (v_\imu^1 - v_\imu^2), \\
    \beta &= b_\imu^2 - b_\imu^1, \\
    B &= (P_\imu^1 T^1)^{-1} (P_\imu^2 T^2), \\
    Q_i ((P_\imu^2 T^2)^{-1} p_i^2) &= (P_\imu^1 T^1)^{-1} p_i^2.
\end{align*}
Then it is straightforward to see that $\phi(X, \xi^1) = \xi^2$.
This completes the proof.
\end{proof}

\begin{proof}[Proof of Lemma \ref{lem:compatible_phi_alpha}]
Let $\xi = (\xi_\imu, b_\imu, T, p_i) \in \vinsT$, $S \in \SE_{\eb_3}(3)$ and $X = (A, \beta, B, Q_i) \in \grpG$ be arbitrary.
Then,
\begin{align*}
    &\phi((A, \beta, B, Q_i), \alpha(S, (\xi_\imu, b_\imu, T, p_i)) \\
    &= \phi((A, \beta, B, Q_i), (S^{-1} P_\imu, R_S^\top v_\imu, b_\imu, T, S^{-1}(p_i))), \\
    &= (S^{-1} P_\imu P_A, R_S^\top v_\imu + R_S^\top R_\imu v_A, b_\imu + \beta, A^{-1} T B,
    \\ &\hspace{1cm}
    S^{-1} P_\imu T B Q_i^{-1} (S^{-1} P_\imu)^{-1} S^{-1}(p_i)
    ), \\
    &= (S^{-1} (P_\imu P_A), R_S^\top (v_\imu + R_\imu v_A), b_\imu + \beta, A^{-1} T B,
    \\ &\hspace{1cm}
    S^{-1} P_\imu T B Q_i^{-1} P_\imu^{-1}(p_i)
    ), \\
    &= \alpha(S, (P_\imu P_A, v_\imu + R_\imu v_A, b_\imu + \beta, A^{-1} T B,
    \\ &\hspace{1cm}
    P_\imu T B Q_i^{-1} P_\imu^{-1}(p_i)
    )), \\
    &= \alpha(S, \phi(X, \xi)),
\end{align*}
as required.
\end{proof}

\begin{proof}[Proof of Lemma \ref{lem:output_action}]
It is trivial to see that $\rho$ is a group action.
To show the equivariance of $h$, one examines the component measurement functions $h^k$.
Let $X = (A, \beta, B, Q_i) \in \grpG$ and $\xi = (\xi_\imu, b_\imu, T, p_i) \in \vinsT$.
Then one has
\begin{align*}
    &h^k(\phi((A, \beta, B, Q_i), (\xi_\imu, b_\imu, T, p_i))) \\
    &= h^k(\varphi^\imu(A, \xi_\imu), b_\imu + \beta,
    P_A^{-1} T B, P_\imu T B Q_i^{-1} T^{-1} P_\imu^{-1} (p_i)), \\
    &= \pi_{\Sph^2}((P_\imu P_A)^{-1} (P_\imu T B Q_k^{-1} T^{-1} P_\imu^{-1} (p_k))), \\
    &= \pi_{\Sph^2}((P_\imu P_A P_A^{-1} T B)^{-1} (P_\imu T B Q_k^{-1} T^{-1} P_\imu^{-1} (p_k))), \\
    &= \pi_{\Sph^2}(Q_k^{-1} T^{-1} P_\imu^{-1} (p_k)), \\
    &= \frac{c_{Q_k}^{-1} R_{Q_k}^\top T^{-1} P_\imu^{-1} (p_k)}{\vert c_{Q_k}^{-1} R_{Q_k}^\top T^{-1} P_\imu^{-1} (p_k) \vert}, \\
    &= R_{Q_k}^\top \frac{T^{-1} P_\imu^{-1} (p_k)}{\vert T^{-1} P_\imu^{-1} (p_k) \vert}, \\
    &= R_{Q_k}^\top h^k(\xi_\imu, b_\imu, T, p_i).
\end{align*}
It follows that
\begin{align*}
    & h(\phi((A, \beta, B, Q_i), (\xi_\imu, b_\imu, T, p_i))) \\
    &= (h^1(\phi((A, \beta, B, Q_i), (\xi_\imu, b_\imu, T, p_i))), ...,
    \\ &\hspace{2cm}
    h^n(\phi((A, \beta, B, Q_i), (\xi_\imu, b_\imu, T, p_i)))), \\
    &= (R_{Q_1}^\top h^1(\xi_\imu, b_\imu, T, p_1), ...,
    R_{Q_n}^\top h^n(\xi_\imu, b_\imu, T, p_n)), \\
    &= \rho((A, \beta, B, Q_i), h(\xi_\imu, b_\imu, T, p_i)),
\end{align*}
as required.
\end{proof}



\bibliographystyle{IEEEtran}
\bibliography{TRO_VIO_2021}

%

\parbox[t]{\linewidth}{
\begin{wrapfigure}{L}{\biographyPhotoWidth}
\includegraphics[width=\biographyPhotoWidth]{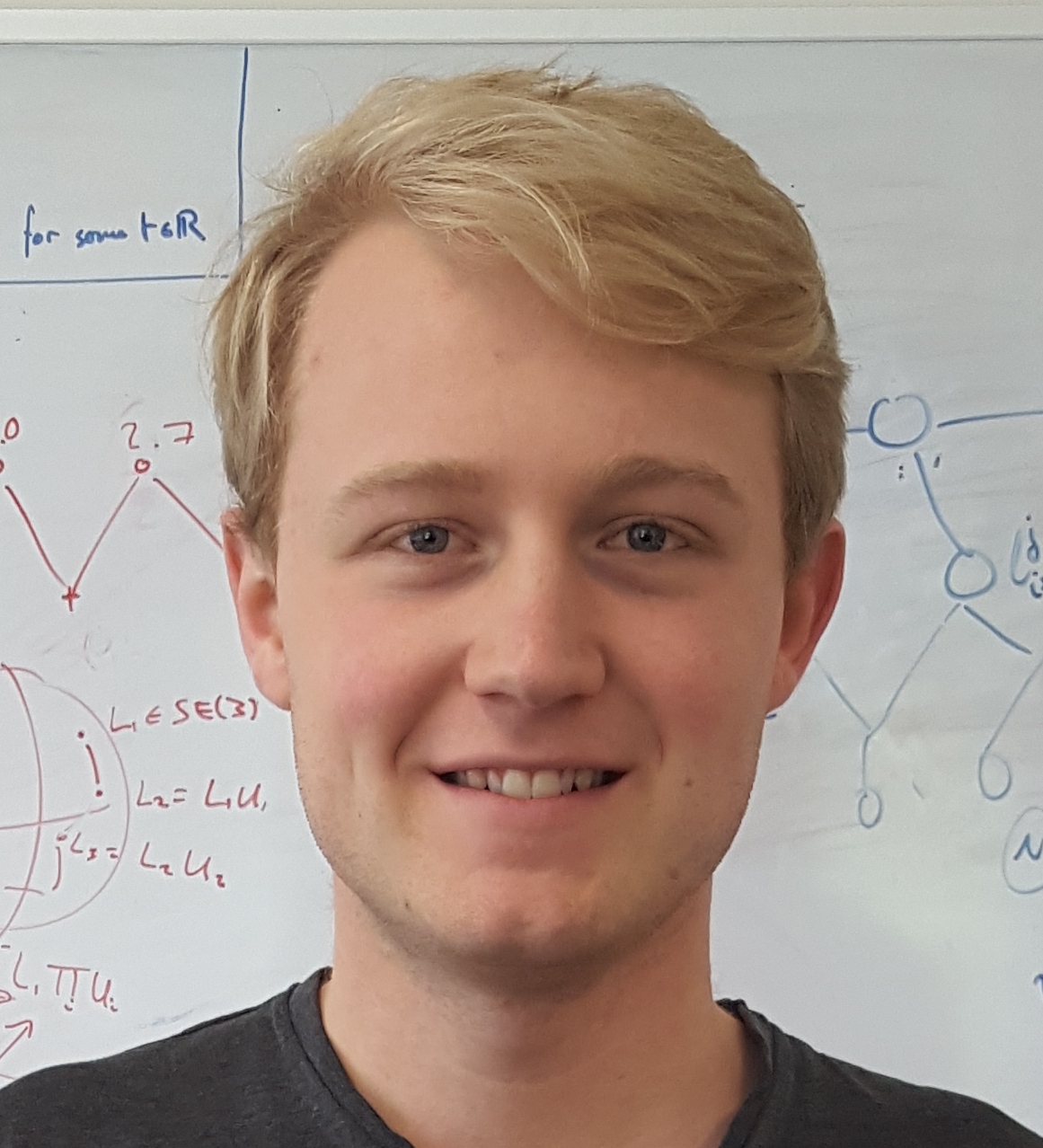}
\end{wrapfigure}
\noindent

\textbf{Pieter van Goor}\
is a PhD student at the Australian National University.
He received his BEng(R\&D) / BSc (Mechatronics and Advanced Mathematics) with first class Honours in 2018 from the Australian National University.
His research interests are in applications of Lie group symmetries and geometric methods to problems in visual spatial awareness.
He is an IEEE student member.
}

\vspace{1.5cm}


\parbox[t]{\linewidth}{
\begin{wrapfigure}{L}{\biographyPhotoWidth}
\includegraphics[width=\biographyPhotoWidth]{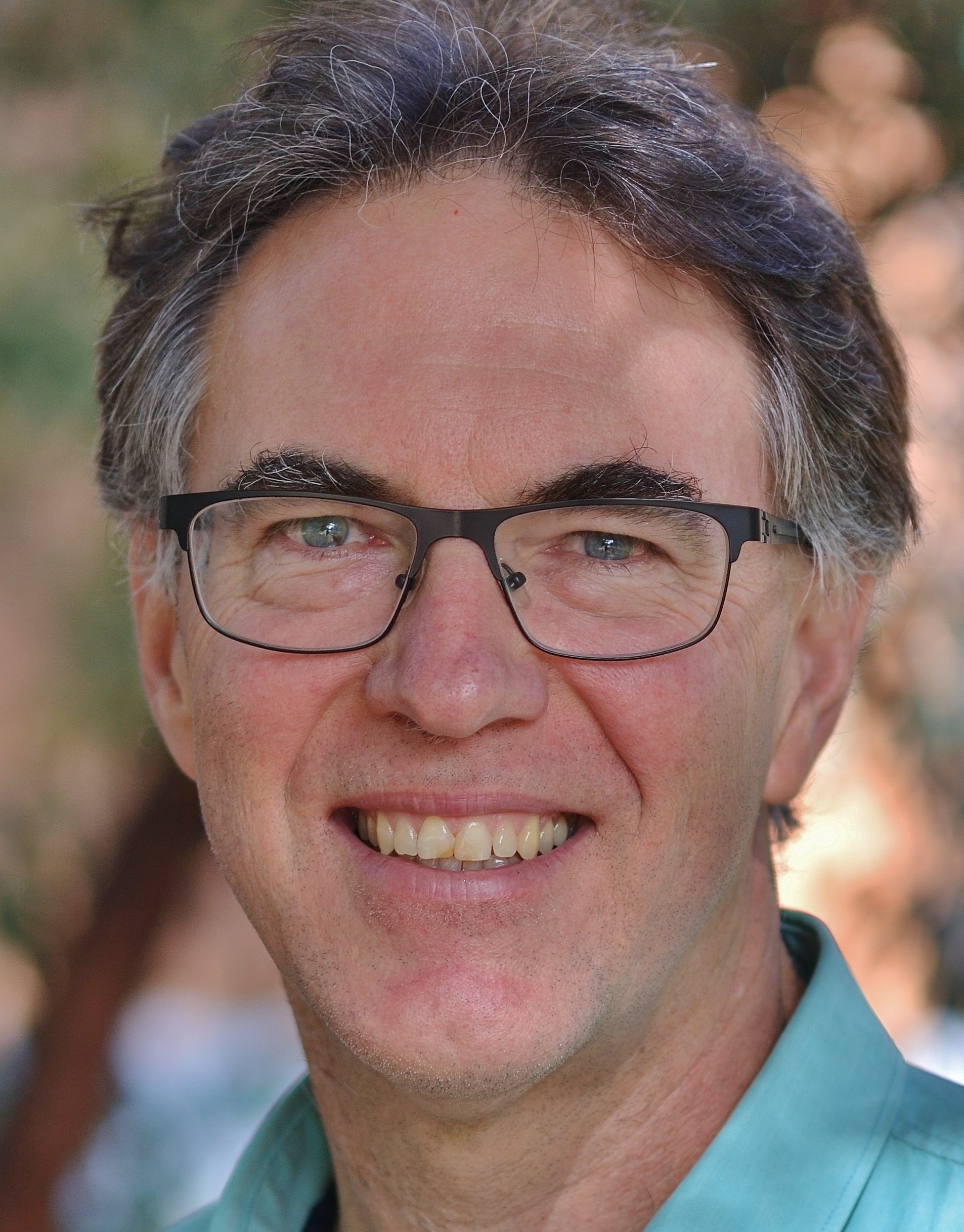}
\end{wrapfigure} \noindent

\textbf{Robert Mahony}%
is a Professor in the School of Engineering at the Australian National University.
He received his BSc in 1989 (applied mathematics and geology) and his PhD in 1995 (systems engineering) both from the Australian National University.
He is a fellow of the IEEE and was president of the Australian Robotics Association from 2008-2011.
He was Director of the Research School of Engineering at the Australian National University 2014-2016.
His research interests are in nonlinear systems theory with applications in robotics and computer vision.
He is known for his work in aerial robotics, equivariant observer design, matrix subspace optimisation and image based visual servo control.
}

\end{document}